\documentclass[arxiv,12pt]{colt2020} 

\usepackage{times}

\usepackage{mathextra}
\usepackage{enumerate}
\usepackage[capitalize]{cleveref}

\makeatletter
\let\Ginclude@graphics\@org@Ginclude@graphics
\makeatother

\title[Information Directed Sampling for Linear Partial Monitoring]{Information Directed Sampling for Linear Partial Monitoring}

\coltauthor{%
 \Name{Johannes Kirschner} \Email{jkirschner@inf.ethz.ch}\\
 \addr ETH Zurich, Department of Computer Science
 \AND
 \Name{Tor Lattimore} \Email{lattimore@google.com}\\
 \addr DeepMind%
 \AND
 \Name{Andreas Krause} \Email{krausea@ethz.ch}\\
 \addr ETH Zurich, Department of Computer Science
}

\let\todo\undefined  
\usepackage[textsize=tiny,disable]{todonotes}
\usepackage{tcolorbox}

\newcommand{\todoj}[2][]{\todo[size=\scriptsize,color=green!20!white,#1]{Johannes: #2}}
\newcommand{\todojg}[2][]{\todo[size=\scriptsize,color=gray!20!white,#1]{Note: #2}}

\newtcbox{\entoure}[1][red]{on line,
	arc=3pt,colback=#1!10!white,colframe=#1!50!black,
	before upper={\rule[-3pt]{0pt}{10pt}},boxrule=1pt,
	boxsep=0pt,left=2pt,right=2pt,top=1pt,bottom=-1pt}

\newcommand{\ip}[1]{\langle #1 \rangle}

\newcommand{\diam}{\operatorname{diam}}

\newcommand{\ucb}{\text{ucb}}

\newtheorem{assumption}{Assumption}

\renewcommand{\a}{a}  

\begin{document}

\maketitle

\begin{abstract}%
Partial monitoring is a rich framework for sequential decision making under uncertainty that generalizes many well known bandit models, including linear, combinatorial  and dueling bandits. We introduce {\em information directed sampling} (IDS) for stochastic partial monitoring with a linear reward and observation structure. IDS achieves adaptive worst-case regret rates that depend on precise observability conditions of the game. Moreover, we prove lower bounds that classify the minimax regret of all finite games into four possible regimes. IDS achieves the optimal rate in all cases  up to logarithmic factors, without tuning any hyper-parameters. We further extend our results to the contextual and the kernelized setting, which significantly increases the range of possible applications.
\end{abstract} \todojg[color=red!50!white]{Disable comments for submission!}


\section{Introduction}
\emph{Partial monitoring} is an expressive framework for sequential decision making in which the learner does not directly observe the reward \citep{Rus99}. Instead, the learner obtains observations from pre-specified observation distributions that are associated to the actions and may or may not provide direct information about the reward.
In this work, we consider a stochastic version of the problem with a linear reward and observation model, which is sometimes referred to as
combinatorial partial monitoring \citep{Lin2014combinatorial,Chaudhuri2016phased}. 
Among other settings as described in Section~\ref{sec:applications}, the linear partial monitoring model strictly generalizes linear bandits \citep{abe1999associative,auer2002confidencebounds}, combinatorial bandits \citep{cesa2012combinatorial} both with bandit and semi-bandit feedback, and some variants of dueling bandits \citep{yue2009interactively}.

\paragraph{Linear Partial Monitoring} Let $\xX \subset \RR^d$ be a compact set of actions and $\theta \in \RR^d$ be an unknown parameter. For each action $x\in \xX$, let $A_x \in \RR^{d \times m}$ be a known \emph{linear observation operator}. 
The learner and environment interact over $n$ rounds. In each round $t$, the learner chooses an action $x_t \in \xX$ and receives an $m$-dimensional observation $\a_t = A_{x_t}^\T\theta + \epsilon_t$ where $(\epsilon_t)_{t=1}^n$ is a sequence of independent $\rho$-subgaussian noise vectors such that $\epsilon_t \in \RR^{m}$. The reward for the learner is $\ip{x_t,\theta}$ and is \emph{not observed}.
As usual, the aim is to minimize cumulative regret
\begin{align*}
R_n = \sum_{t=1}^n \ip{x^* - x_t,\theta}\,,
\end{align*}
where $x^* = \argmax_{x \in \xX} \ip{x,\theta}$ is the optimal action, chosen arbitrarily whenever the choice is not unique. A slightly more general formulation of the setup is in Appendix \ref{app:general-setup}, which we will use for some applications. \emph{Bandit games} are a special case where $A_x = x$. We discuss further applications in detail in Section \ref{sec:applications}. 
Readers seeking further motivation and intuition for the setup will benefit from skipping ahead to this section.

A linear partial monitoring game is called {\em finite} if it has finitely many actions.
An action $x \in \xX$ is called {\em Pareto optimal} if 
it is an extreme point of the convex hull of $\xX$.
The set of actions that are optimal for $\theta$ is
\begin{align*}
\pP(\theta) = \{x  \in \xX : \ip{x, \theta} = \max_{y \in \xX} \ip{y, \theta}\}\,,
\end{align*}
which is defined on sets $\cC \subseteq \RR^d$ by $\pP(\cC) = \cup_{\theta \in \cC} \pP(\theta)$. A game is called {\em globally observable} if
\begin{align}
x - y \in \laspan(A_z : z \in \xX) \text{ for all } x, y \in \xX\,. \label{eq:global}
\end{align}
A game is called {\em locally observable} if for every convex set $\cC \subset \RR^d$,
\begin{align}
x - y \in \laspan(A_z : z \in \pP(\cC)) \text{ for all } x, y \in \pP(\cC)\,.  \label{eq:local}
\end{align}
\looseness=-1 Intuitively, in globally observable games, the learner has access to actions from which reward differences between different actions can be estimated. In locally observable games, the reward differences can be estimated
in a local sense that greatly eases learning. Any locally observable game is also globally observable.
Although it is not important for this work, connoisseurs of partial monitoring will be pleased to know these definitions coincide with the usual definitions, as discussed in
\cref{app:class}.
At least for finite games, we will see that global observability leads to a $\tilde O(n^{2/3})$ regret, while local observability leads to $\tilde O(n^{1/2})$ regret. The relation to finite partial monitoring is subtle, however, and the classification results do not imply each other as we explain in Appendix \ref{app:finite-pm}.

\paragraph{Information Directed Sampling} \looseness=-1 We propose a new algorithm  for stochastic linear partial monitoring based on the {\em information directed sampling (IDS)} principle. This strategy uses the observations $\a_t = A_{x_t}^\T \theta + \epsilon_t$ to construct conservative estimates $\Delta_t(x) \geq \ip{x^* - x, \theta}$ of the true gaps and an associated {\em information gain} $I_t(x)$, detailed below. The information gain quantifies the uncertainty reduction in the parameter estimate when the learner chooses $x \in \xX$ and observes $A_x^\T \theta + \epsilon$. 
IDS is the policy that samples action $x_t$ from a distribution $\mu_t$ that minimizes the \emph{information ratio}, 
\begin{align*}
\mu_t  = \argmin_{\mu} \frac{\EE_\mu[\Delta_t(x)]^2}{\EE_\mu[I_t(x)]} \text{ .}
\end{align*}

\paragraph{Our Contributions}
\looseness -1  Our main contribution is a new algorithm for linear stochastic partial monitoring. We show that its regret dependence on the horizon is near-optimal in
all finite-action games without the need to tune any hyper-parameters. Along the way, we prove a classification theorem showing that, up to logarithmic factors, the minimax regret of all finite-action games is either $0$, $\tilde \Theta(n^{1/2})$, $\tilde \Theta(n^{2/3})$ or $\Omega(n)$. 
This result mirrors that for the standard setting \citep{lattimore2019cleaning}, but neither result implies the other. Our upper bounds are general and apply beyond the finite case. For infinite actions, however, the
classification theorem is no longer so straightforward: we show that the minimax regret depends on finer geometric properties of the action set and observation structure, such as curvature. We further consider a novel {\em contextual partial monitoring setting}, where IDS exhibits an elegant planning behavior to exploit the  distribution over contexts. Lastly, our algorithm and analysis are easily {\em kernelized}, which enables utilizing practically important smoothness priors, with applications such as Bayesian optimization with gradient observations.

\paragraph{Related work}
\looseness -1 Finite partial monitoring dates back to \cite{Rus99}. The generality of partial monitoring yields a rich structure of games \citep{bartok2014partial,lattimore2019cleaning} where the minimax regret rate depends on precise observability conditions. The complete classification of finite games is achieved in a line of work by \citet{CBLuSt06,bartok2014partial,ABPS13,lattimore2019cleaning}, with a focus on the stochastic version of the problem in the work by \citet{BSP11,BZS12}. Asymptotics for finite games are known as well \citep{KHN15}. Partial monitoring with prior information was studied by \cite{vanchinathan14efficient}, and  with side-information by \citet{bartok2012partial}. Latter setup is different from our contextual setting. The linear version of the problem that we study here is due to \citet{Lin2014combinatorial,Chaudhuri2016phased}. Both previous approaches rely on forced exploration schemes and achieve a $\tilde \oO(n^{2/3})$ worst-case regret on globally observable games, but not the faster $\tilde \oO(n^{1/2})$ rate on locally observable games. Information directed sampling was proposed by \citet{Russo2014learning} in the Bayesian setting to address short-comings of the UCB algorithm \citep{auer2002confidencebounds} and Thompson sampling \citep{agrawal2013thompson} on examples that capture the spirit of partial monitoring. The frequentist version of the algorithm that we analyze here was proposed by \citet{Kirschner2018} in a bandit setting with heteroscedastic noise, which we strictly generalize. Recently, information theoretic tools were also introduced in the partial monitoring literature, to obtain minimax rates \citep{lattimore2019information} and to define the sampling distribution of an algorithm for finite (adversarial) partial monitoring \citep{lattimore2019exploration}.  Bandits \citep{lattimore2018bandit} are perhaps the most prominent special case of partial monitoring. We discuss more relevant work in the context of specific applications in Section \ref{sec:applications}.

\paragraph{Notation}
We write $\norm{\cdot}$ for the standard euclidean norm.  For positive semi-definite $A$ let $\norm{x}_A^2 = x^\T A x$.
Given $\xX \subset \RR^d$ we let $\conv(\xX)$, $\partial \xX$ and $\diam(\xX) = \sup_{x,y \in \xX} \norm{x - y}$ denote its convex hull, boundary and diameter respectively.
The smallest and largest eigenvalues of a matrix $A$ are denoted by $\lambda_{\min}(A)$ and $\lambda_{\max}(A)$ respectively.
The identity matrix of dimension $d$ is denoted by $\eye_d$. For two square matrices $A,B$, $A \preceq B$ means that $B-A$ is positive semi-definite. For a possibly non-square matrix $C$, $\|C\|_2 = \sqrt{\lambda_{\max}(C^\T C)}$ is the operator norm. Given an index set $\zZ$ and a collection of matrices $(A_z : z \in \zZ)$, all with the same number of rows, we define $\laspan(A_z : z \in \zZ)$ to be the span of the collection of all columns of the matrices $(A_z : z \in \zZ)$.
When $\xX$ is Borel measurable we let $\sP(\xX)$ be the space of probability measures on $\xX$ with respect to the Borel $\sigma$-algebra.
The Dirac probability measure at $x \in \xX$ is denoted by $\delta_x \in \sP(\xX)$.
The optimal action given parameter $\theta$ is $x^*(\theta) = \argmax_{x \in \xX} \ip{x,\theta}$.
Probability measures on subsets of $\RR^d$ are always defined over the Borel $\sigma$-algebra.
Given a probability measure $\pi$ on $\xX$ let $V_\pi = \int_{\xX} A_x A_x^\T d\pi(x) \in \RR^{d \times d}$.
The filtration $\fF_t = \sigma(x_1, \a_1, \dots, x_t, \a_t)$ contains the observed quantities at the end of round $t$. $\tilde \oO$ is the Landau notation with logarithmic factors suppressed.

\paragraph{Assumptions}
Throughout, we make technical boundedness assumptions $\|\theta\|_2 \leq 1$, $\|A_x\|_2 \leq 1$ and $\diam{(\xX)}\leq 1$, which implies $\ip{x-y,\theta} \leq 1$ for all $x,y \in \xX$. 
The noise vector $\epsilon_t$ is conditionally $\rho$\nobreakdash-subgaussian, $\forall \lambda \in \RR,\,\EE[e^{\lambda \epsilon_t}|\fF_{t-1}, x_t] \leq \exp(\rho^2 \lambda^2/2)$ understood coordinate wise. The map $x \mapsto A_x$ is assumed to be continuous.

\section{Information Directed Sampling for Linear Partial Monitoring}
{\em Information directed sampling (IDS)} was introduced by \citet{Russo2014learning} for the Bayesian bandit setting. 
IDS samples actions from a distribution that minimizes the ratio of squared expected regret and mutual information. The \emph{information ratio} 
appears in a sum under the square root in the regret bound and IDS is the policy that (greedily) minimizes this term. \citet{Kirschner2018} introduced a frequentist analog of the algorithm 
that replaces the Bayesian expected suboptimality and information gain with frequentist counterparts and exhibits high-probability regret bounds on linear bandits. In the following we generalize the latter approach, which we simply refer to as IDS. 

Formally, let $\Delta_t:\xX \rightarrow \RR_{\geq 0}$ be a \emph{gap estimate} and $I_t : \xX\rightarrow \RR_{\geq 0}$ an \emph{information gain} that we will define shortly. We linearly extend the functions $\Delta_t$ and $I_t$ to probability distributions over $\xX$ so that for distributions $\mu \in \sP(\xX)$ we have
$\Delta_t(\mu) = \int_\xX \Delta_t(x) d\mu(x)$ and $I_t(\mu) = \int_\xX I_t(x) d\mu(x)$.
Information directed sampling is the strategy that samples the action $x_t$ at step $t$ from a distribution $\mu_t \in \sP(\xX)$ that minimizes the \emph{information ratio} $\Psi_t(\mu)$: 
\begin{align*}
\mu_t = \argmin_{\mu \in \sP(\xX)} \Psi_t(\mu)\,, \quad \text{where} \quad \Psi_t(\mu) = \frac{\Delta_t(\mu)^2}{I_t(\mu)}\,.
\end{align*}
The minimizing distribution is well defined and can 
always be chosen with a support of two actions. Furthermore, $\mu \mapsto \Psi_t(\mu)$ is convex. These results were previously shown for the bandit setting by \citet[Lemma 4,16,17]{Kirschner2018} and continue to hold in the more general setting. We briefly discuss computational concerns in Section \ref{sec:discussion}.

Let $A_t = A_{x_t}$ be the observation operator for the action chosen in round $t$. To estimate the gap $\Delta_t(x)$, IDS uses the regularized least squares estimator, which after $t$ rounds is
\begin{align*}
\hat{\theta}_t = \argmin_{\theta \in \RR^d} \sum_{s=1}^t \snorm{A_s^\T \theta - \a_s}^2 + \snorm{\theta}^2 = V_t^{-1}\sum_{s=1}^t A_s \a_s\,,
\end{align*}
where $V_t = \sum_{s=1}^t A_sA_s^\T + \eye_d$. Define a sequence of confidence sets $(\cC_t)_{t=0}^n$ by
\begin{align}
\cC_t = \left\{\theta' \in \RR^d : \snorm{\hat \theta_t - \theta'}_{V_t} \leq \beta_t^{1/2}\right\}\,, \quad \text{where }
\beta_t^{1/2} = \sqrt{\log \det V_t + 2\log\left(\tfrac{1}{\delta}\right)} + 1\,.
\label{eq:conf}
\end{align}
The concentration bound by \citet[Theorem 2]{Abbasi2011improved} shows that with probability at least $1 - \delta$ it holds that $\theta \in \cC_t$ for all $t$. Our estimate of the suboptimality gap is defined as
\begin{align*}
\Delta_t(x) = \min\Big\{1, \, \max_{y \in \xX} \ip{y - x,\hat \theta_{t-1}} + \beta_{t-1}^{1/2} \norm{x - y}_{V_{t-1}^{-1}}\Big\}\,,
\end{align*}
which is chosen so that with high probability $\ip{x^* - x,\theta} \leq \Delta_t(x)$ for all $x \in \xX$ and all rounds $t$.
Note that $\Delta_t(x) \geq 0$ for all $x \in \xX$.
For the information gain we use
\begin{align}
I_t(x) = \log \det \left(\eye_{m} + A_x^\T V_{t-1}^{-1} A_x\right)\,. \label{eq:info-full}
\end{align}
The definition corresponds to the usual Shannon mutual information when using a Gaussian prior on the parameter and a Gaussian likelihood function. For the bandit setting, it was previously demonstrated by \cite{Kirschner2018} that the choice of $I_t$ can have a large impact on empirical performance. In Appendix \ref{app:other-info} we discuss some alternative choices for both $\Delta_t$ and $I_t$.

\subsection{A General Regret Bound}
The regret of any strategy can be bounded in terms of the cumulative sum of the information ratio $\sum_{t=1}^n \Psi_t(\mu_t)$ 
and the \emph{total information gain}, $\gamma_n = \sum_{t=1}^n I_t(x_t)$. The following result is a generic regret bound that generalizes Theorem 1 of \citet{Kirschner2018}. Note that for deterministic policies the result can be simplified \cite[cf.\ Theorem 2]{Kirschner2018}.

\begin{lemma}[IDS regret bound]\label{lem:ids-regret}
For $\eta \geq 0$, let $G_\eta = \{t \in [n] : \Delta_t(\mu_t) \leq \eta\}$.
There exists a universal constant $C > 1$ such that for any $n \geq 1$ with probability at least $1 - \delta$ the regret of any (possibly) randomized policy $(\mu_t)_{t=1}^n$ is bounded by 
\begin{align*}
R_n \leq C \inf_{\eta \geq 0} \left[n\eta + \sqrt{\sum_{t\in [n]\setminus G_\eta} \Psi_t(\mu_t)\left(\gamma_n + \log\tfrac{1}{\delta}\right)}\right] + 4\log\left(\frac{4n+4}{\delta}\right)\,.
\end{align*}
\end{lemma}
For the proof, note that $\ip{x^* - x_t, \theta} \leq \Delta_t(x_t)$ and consider the sum over the expected gap estimates,
\begin{align*}
\sum_{t=1}^n \Delta_t(\mu_t)
&= \sum_{t\in G_\eta} \Delta_t(\mu_t)  +  \sum_{t\in [n]\setminus G_\eta} \Delta_t(\mu_t) 
\leq n\eta + \sqrt{\sum_{t\in [n]\setminus G_\eta} \Psi_t(\mu_t) \sum_{t=1}^n I_t(\mu_t)}\,, 
\end{align*}
where we used the definition of $G_\eta$ and the Cauchy--Schwarz inequality. A variance-dependent martingale bound such as Freedman's inequality shows that the regret $R_n$ concentrates on the sum over (conditional) expected regret up to an additive $\oO(\log n)$ term and the total expected information gain is bounded by $\sum_{t=1}^n I_t(\mu_t) \leq c(\gamma_n + \log \frac{1}{\delta})$. The complete proof is given in Appendix \ref{app:proof-ids-regret}.

The next lemma is a standard result \citep[cf.\ Lemma 10]{Abbasi2011improved} and shows that for fixed dimension, the total information gain depends only logarithmically on the horizon.
\begin{lemma}\label{lem:total-info-gain}
	For the information gain $I_t$ as defined in \eqref{eq:info-full}, $\gamma_n = \log\det(V_n) - \log \det(V_0) \leq d\log\left(1 + \tfrac{nm}{d}\right)$ and $\beta_n = \sqrt{\gamma_n + 2 \log \tfrac{1}{\delta}} + 1$.
\end{lemma}
Importantly, $\gamma_n$ and $\beta_n$ have no dependence on the number of actions. Further, if $\theta \in \hH$ is contained in a reproducing kernel Hilbert space $\hH$ with bounded Hilbert norm $\|\theta\|_{\hH} \leq 1$ (corresponding to a Gaussian process), even the dependence on $d$ can be avoided \citep{Srinivas2009}.

\subsection{Regret Bound for Globally Observable Games}
We first analyze globally observable games (see \cref{eq:global}). The condition implies that $\ip{x - y, \theta}$ can be estimated from data collected by the algorithm using appropriate actions.
The game-dependent constants that appear in the analysis depend on the degree to which the learner can efficiently gain information, which roughly depends on how well the observation operators $A_x$ are 
aligned with a direction $x-y$ in which we try to improve the accuracy of our estimation.
 We define the  \emph{worst-case alignment} constant as
\begin{align*}
\alpha = \max_{v \in \RR^d} \max_{x,y \in \xX}  \min_{z \in \xX} \frac{\ip{x-y,v}^2}{\|A_z^\T v\|^2}\,.
\end{align*}
Note that for games that are globally observable, $\alpha$ is always bounded, independent of the number of actions (Lemma \ref{lem:alignment-bounds}, Appendix \ref{app:alignment-bounds}).

\begin{theorem}\label{thm:global-upper}
	For any game that satisfies the global observability condition \eqref{eq:global}, there exists a universal constant $C > 0$ such that for any $n\geq 1$ with probability at least $1 - \delta$,
	\begin{align*}
	R_n \leq C n^{2/3} \left(\alpha \beta_n (\gamma_n + \log\tfrac{1}{\delta})\right)^{1/3} + 4\log\left(\frac{4n+4}{\delta}\right)\,.
	\end{align*}
\end{theorem}

For fixed feature dimension, $\beta_n$ and $\gamma_n$ depend only logarithmically on the horizon, and therefore the regret for globally observable games is $R_n \leq \tilde \oO(n^{2/3})$. We show in \cref{sec:lower} that all
globally observable games that are not locally observable have $R_n = \Omega(n^{2/3})$ in the worst case. 
The first step in the proof of Theorem~\ref{thm:global-upper} is to establish the existence of an action for which the information gain
is large relative to the regret of the greedy action.

\begin{lemma}\label{lem:global-exp}
	The greedy action $\hat x^*_t = \argmax_{x \in \xX} \ip{x,\hat \theta_{t-1}}$ satisfies
	$\Delta_t(\hat x^*_t)^2 \leq 2 \alpha \beta_{t-1} \max_z I_t(z)$.
\end{lemma}
\begin{proof}
	Let $w_t = \argmax_{w \in \{x-y : x,y \in \xX\}} \|w\|_{V_{t-1}^{-1}}^2$ be the most uncertain direction. Then, 
	\begin{align*}
	\Delta_t(\hat x^*_t) = \max_{y \in \xX} (y - \hat x^*_t)^\T \hat \theta_{t-1} + \beta_{t-1}^{1/2} \snorm{y - \hat x_t^*}_{V_{t-1}^{-1}}
	\leq \beta_{t-1}^{1/2}\|w_t\|_{V_{t-1}^{-1}}\,.
	\end{align*}
  	Note three ways to write $\|w\|_{V_{t-1}^{-1}}^2 = \| V_{t-1}^{-1/2} w_t \|^2 = \ip{w_t, V_{t-1}^{-1}w_t}$. Basic linear algebra shows that
	\begin{align*}
	\frac{\|A_z^\top V_{t-1}^{-1} w_t \|^2}{ \| V_{t-1}^{-1/2} w_t \|^2}
	&\leq \max_{v \in \RR^d} \frac{\|A_z^\top V_{t-1}^{-1/2}v\|^2}{ \| v \|^2} = \lambda_{\max}(A_z^\T V_{t-1}^{-1} A_z)\\
	& \leq 2 \log\det(\eye_{d}+ A_z^\T V_{t-1}^{-1}A_z) = 2 I_t(z)\,.
	\end{align*}
	For the last step we used the inequality $a \leq 2 \log(1+a)$ for $a \in [0,1]$ and that the eigenvalues of $A_z^\T V_{t-1}^{-1} A_z$ are bounded in $[0,1]$ by the assumption $\|A_z\| \leq 1$ and $V_{t-1}^{-1} \preceq \eye_{d}$. With the most informative action $z_t = \argmax_{x \in \xX} I_t(x)$, it follows that
	\begin{align*}
	\frac{\Delta_t(\hat x_t^*)^2}{I_t(z_t)} \leq 2 \beta_{t-1} \min_z\frac{\ip{w_t, V_{t-1}^{-1} w_t}^2}{\|A_z^\T V_{t-1}^{-1} w_t \|^2 } \leq 2 \beta_{t-1}\max_{v \in \RR^d} \max_{x,y \in \xX} \min_z \frac{\ip{x-y, v}^2}{\|A_z^\T v \|^2 } = 2 \alpha \beta_{t-1}\,.
	\end{align*}
	Rearranging completes the proof.
\end{proof}%

The following lemma shows that IDS never plays a distribution that is too far from greedy. The proof is deferred to Appendix~\ref{app:proof-greedy}.

\begin{lemma}\label{lem:greedy}
  Let $\mu_t$ be the IDS distribution at time $t$. Then $\Delta_t(\mu_t) \leq 2 \min_{x \in \xX} \Delta_t(x)$.
\end{lemma}

\begin{proof}\textbf{of \cref{thm:global-upper}}\,\,
  Let $z_t = \argmax_{x \in \xX} I_t(x)$ be the informative action. 
	For $p \in [0,1]$, let $\mu(p) = (1 - p) \delta_{\hat x^*_t} + p \delta_{z_t}$ be the distribution that randomizes between the greedy and the informative action. 
  By definition, the information ratio of IDS is bounded by the ratio of $\mu(p)$,
	\begin{align*}
	\Psi_t(\mu_t) \leq \min_{p \in [0,1]} \frac{\Delta_t(\mu(p))^2}{I_t(\mu(p))}
	\leq  2 \alpha \beta_{t-1} \min_{p \in [0,1]}  \frac{\left((1 - p) \Delta_t(\hat x^*_t) + p\right)^2}{p \Delta_t(\hat x_t^*)^2} \leq \frac{8 \alpha \beta_{t-1}}{\Delta_t(\hat x_t^*)}
  \leq \frac{16 \alpha \beta_{t-1}}{\Delta_t(\mu_t)}\,.
	\end{align*}
	The second inequality uses $\Delta_t(z) \leq 1$, $I_t(\hat x_t^*) \geq 0$ and Lemma~\ref{lem:global-exp} to bound $\Delta_t(\hat x^*_t)^2 \leq 2 \alpha \beta_{t-1} I_t(z)$. 
  The third inequality follows by choosing $p = \Delta_t(\hat x_t^*) \in [0,1]$ and the last follows from Lemma~\ref{lem:greedy}.
  Next, Lemma \ref{lem:ids-regret} shows that with probability at least $1 - \delta$,
	\begin{align*}
	R_n \leq C \inf_{\eta \geq 0} \left[n \eta + \sqrt{\frac{16 n \alpha \beta_n (\gamma_n + \log \tfrac{1}{\delta})}{\eta}}\right] + 4\log\left(\frac{4n+4}{\delta}\right)\,,
	\end{align*}
	where we used the fact that $(\beta_t)_{t=0}^n$ is non-decreasing.
	Optimizing $\eta$ completes the proof.
\end{proof}

\subsection{Regret Bound for Locally Observable Games}
In globally observable games, the learner can estimate the gaps for all actions, but may need to play actions that are known to be suboptimal.
The definition of local observability (see Eq.\ \eqref{eq:local}) means that the learner can gain information while playing only actions that appear plausibly optimal.

\looseness=-1 Recall the definition of the confidence set $\cC_t$ in \cref{eq:conf} and let $\pP_t = \pP(\cC_{t-1})$ be the set of actions that are plausibly optimal in round $t$. Again, our bound depends on the signal to noise ratio when exploring. For a set of (plausible optimal) actions $\yY \subset \xX$, define the \emph{worst-case alignment} for $\yY$,
\begin{align}
 \alpha(\yY) = \max_{v \in \RR^d} \max_{x,y \in \yY} \min_{z \in \yY} \frac{\ip{x-y,v}^2}{\|A_z^\T v\|^2}\,.
\end{align}
Globally observable games satisfy $\alpha = \alpha(\xX) < \infty$. The local observability condition implies that this remains true if we restrict actions to $\pP(\cC)$. All games with bandit feedback ($A_x = x$) satisfy $\alpha(\yY) \leq 4$. We refer to Lemma \ref{lem:alignment-bounds} in Appendix \ref{app:alignment-bounds} for details.
We say a game is \emph{uniformly locally observable} if $\alpha(\pP(\cC)) \leq \alpha_0$ for all $\cC \subset \RR^d$ convex. All finite locally observable games are uniformly locally observable because there are only finitely many subsets.
The definition of the alignment constant can be tightened with a more careful analysis, to obtain improved bounds on model parameters such as the dimension in some cases. We refer to Appendix \ref{app:dueling} for details.
\begin{theorem}\label{thm:local-upper}
	For locally observable games denote $\alpha_t = \alpha(\pP_t)$. There exists a universal constant $C > 0$ such that for any $n \geq 1$ with probability at least $1-\delta$,
	\begin{align*}
	R_n \leq C \sqrt{\sum_{t=1}^n\alpha_t \beta_t \left( \gamma_n + \log \tfrac{1}{\delta} \right)} + 4\log\left(\frac{4n+4}{\delta}\right) \,.
	\end{align*}
	For games that are uniformly locally observable, the regret bound is $R_n \leq \tilde \oO(\sqrt{\alpha_0 \beta_n \gamma_n n})$.
\end{theorem}

We show in Appendix \ref{sec:lower} that on locally observable games with more than one Pareto optimal action, any algorithm suffers $\Omega(n^{1/2})$ regret in the worst case. To prove the upper bound, the first step is to construct an exploration distribution that is supported on the plausible maximizers $\pP_t$ and has a constant information ratio. 
Note that IDS is not restricted to playing actions within $\pP_t$, nor is it required to explicitly compute this set. In fact, actions that are not plausible maximizers can have a better trade-of between regret and information. 

\begin{lemma}\label{lem:local-exp}
  For locally observable games, there exists an exploration action in $\pP_t$ such that,
	\begin{align*}
	\forall x \in \pP_t\,,\quad \Delta_t(x)^2 \leq 8 \alpha_t \beta_{t-1} \max_{z\in \pP_t} I_t(z) \,.
	\end{align*}
\end{lemma}
The complete proof is in Appendix \ref{app:proof-local-exp}. The argument shows that  for plausible maximizers $x \in \pP_t$, $\Delta_t(x)^2 \leq 4 \beta_{t-1} \max_{y \in \pP_t} \|y-x\|_{V_{t-1}^{-1}}^2$  and is otherwise similar to the proof of Lemma \ref{lem:global-exp}.

\begin{proof}\textbf{of Theorem \ref{thm:local-upper}}
	Let $z_t = \argmax_{x \in \pP_t} I_t(x)$ be the most informative action in the current plausible maximizer set $\pP_t$. By Lemma \ref{lem:local-exp},
	\begin{align*}
	\Psi_t(\mu_t) \leq \frac{\Delta_t(z_t)^2}{I_t(z_t)} \leq 8\alpha_t \beta_{t-1} \,.
	\end{align*}
	Invoking the general IDS bound (Lemma \ref{lem:ids-regret}) with $\eta =0$ completes the proof.
\end{proof}

\looseness=-1 The proof shows that randomization is not necessary to achieve a bounded information ratio in locally observable games. Deterministic IDS \citep{Kirschner2018}, which optimizes the ratio over a deterministic action choice $x_t = \argmin_{x \in \xX} \Psi_t(\delta_x)$, achieves the same upper bound with our analysis. Moreover, the bound shows how IDS adapts towards the current instance of the partial monitoring game. Consider a globally observable game where after some finite time $n_0$, the plausible maximizer sets $\pP_t$ are locally observable in sense that $\alpha(\pP_t) \leq \alpha_0, \forall \, t \geq n_0$. In this case the regret bound is $R_n \leq \tilde \oO(\alpha^{1/3} n_0^{3/2} + (\alpha_0 n)^{1/2})$. We have not yet identified non-artificial conditions that ensure this behavior, however.
The gold standard would be to prove finite-time, instance-dependent regret bounds with small constants. At present such results are more or less restricted to finite-armed bandits, however, and remain open even for linear bandits \citep{hao2019adaptive}.

\subsection{Smooth Convex Action Sets}\label{sub:convex}

The observability conditions are more ambiguous when $\conv(\xX)$ is not a polytope. Here we prove that when $\xX$ has strictly positive principle curvature, 
then IDS enjoys $\tilde O(\sqrt{n})$ regret on globally observable games. Curvature of the action set has been exploited in online learning \citep{HLG17} and bandits \citep{BCL17}.
The latter article considers the starved adversarial linear bandit, where the learner only observes the rewards when sampling an action from a pre-specified
distribution. They consider the case where the action set is the unit ball with
respect to $\norm{\cdot}_p$ and prove that for $p = 2$ one can obtain $O(n^{1/2})$ regret, but not for $p > 2$. This setting is close to a special case of linear partial monitoring (see Appendix \ref{app:starved}). 
Let $h_{\xX} : \RR^d \to \RR$ be the support function of $\xX$, which is defined by $h_{\xX}(u) = \sup_{x \in \xX} \ip{x, u}$.

\begin{theorem}\label{thm:convex}
Assume that $\xX$ is closed, convex, has a non-empty interior and that $h_{\xX}$ is twice differentiable.
Suppose furthermore that the game is globally observable according to \cref{eq:global} and has strictly positive principle curvature everywhere:
\begin{align*}
\kappa_{\circ} = \left(\max_{\eta \in \RR^d : \norm{\eta}_2 = 1} \lambda_{\max} (\nabla^2 h_{\xX}(\eta))\right)^{-1} > 0\,.
\end{align*}
Then, with probability at least $1 - \delta$, for any $n \geq 1$,
\begin{align*}
R_n \leq C\sqrt{\max\left(1, \frac{1}{\kappa_{\circ}}\right) \beta_n \left(\gamma_n + \log\tfrac{1}{\delta}\right)n} + 4 \log\left(\frac{4n + 4}{\delta}\right)\,, 
\end{align*}
where $C$ is a constant depending only on $(A_z : z \in \xX)$.
\end{theorem}

The proof is given in \cref{app:convex}.
The key argument shows that $\Delta_t$ applied to the empirically optimal action scales like the square of the diameter of the confidence set. This
compares favorably with the case without curvature, where the error is about linear in the diameter of the confidence set.

\subsection{Contextual Partial Monitoring Games}\label{sub:context}
The contextual bandit problem is a well known extension of the bandit setting where the learner receives a context before choosing the action \citep{Woo79,LZ08}.
We introduce a novel contextual variant of linear partial monitoring, that strictly generalizes the linear contextual bandit setting.
Let $\zZ$ be a compact set of contexts. Each context $z \in \zZ$ defines a partial monitoring game with action set $\xX_z$ and the observation operators $\{A_x^z\}$, where the map $(x,z) \mapsto A_x^z$ is assumed to be continuous. At time $t$, the learner receives a context $z_t \in \zZ$ and chooses an action $x_t \in \xX_{z_t}$. The reward is $\ip{x_t, \theta}$ and the observation is $\a_t = A_{x_t}^{z_t\T} \theta + \epsilon_t$ where the parameter $\theta$ is the same in every context. The objective is to compete with the best in-hindsight policy that maps context to actions.  Regret is defined with respect to the context-dependent solution $x_t^* = \argmax_{x \in \xX_{z_t}} \ip{x,\theta}$:
\begin{align*}
	R_n = \sum_{t=1}^n \ip{x_t^* - x_t, \theta}\,.
\end{align*}
The regret of the learner depends on the sequence of contexts observed and the corresponding sequence of partial monitoring games which share the common parameter $\theta$. 
All our notions extend with the contextual argument, 
\begin{align*}
\Delta_t(x,z) = \max_{y \in \xX_z} \ip{y-x,\hat{\theta}_t} + \beta_{t-1}^{1/2} \|y-x\|_{V_{t-1}^{-1}}\,, \quad I_t(x,z) = \log \det(\eye + A_x^{z\T} V_{t-1}^{-1} A_x^z)\,.
\end{align*}
\emph{Conditional IDS} is the policy that minimizes the $\mu_t = \argmin_{\mu \in \sP(\xX)} \Psi(\mu,z_t)$ conditioned on the observed context. The next result  extends the regret guarantees for locally and globally observable games to the contextual setting by making strong assumptions on the sequence of games defined by the context. We refer to  Appendix \ref{app:conditional-ids} for our formal result.
\begin{corollary}\textit{(Informal)} If the sequence of games defined by the observed contexts $z_1, \dots, z_n$ are globally observable, conditional IDS achieves $R_n \leq \oO\big(n^{2/3} (\alpha \beta_n (\gamma_n + \log \tfrac{1}{\delta}))^{1/3}\big)$ regret with high probability. If the sequence of games is uniformly locally observable, then conditional IDS achieves $R_n \leq \oO\left((\alpha_0 \beta_n (\gamma_n + \log\frac{1}{\delta}) n)^{1/2}\right)$.
\end{corollary}

Perhaps surprisingly, the contextual case allows for much weaker conditions under which no-regret is possible if the learner exploits the distribution of contexts. 
Here we study the case where the context follows a known distribution $\nu \in \sP(\zZ)$; the case where the distribution is unknown or the learner tries to adapt her behaviour towards an arbitrary sequence is left as an interesting direction for future work. It is instructive to think about some examples:
\begin{itemize}
	\itemsep0pt
	\item An extreme case is where for some $z \in \zZ$ the learner obtains no information ($A_x^z = 0$ for all $x \in \xX_z$). In such rounds the only sensible choice is the greedy action. Exploration needs to happen in rounds where information is available and needs to be sufficiently diverse to account for rounds where the learner is forced to play greedily. Note that while there can be vanishing information gain in {\em some} rounds, the {\em expected} information gain, that takes the distribution $\nu$ over the context into account, is non-zero.
	\item Since also the greedy action depends on the random context, there can be cases where the learner incurs sufficient exploration by playing mostly greedy. This effect has been studied in the bandit literature before \citep{bastani2017mostly,hao2019adaptive}.
\end{itemize}
Conditional IDS does not depend on the distribution $\nu$ and it is easy to see that it can behave suboptimally in both examples. To include the randomness of the context within the IDS framework, consider a joint distribution $\xi \in \sP(\xX \times \cC)$ over context and actions with marginal $\xi_z \in \sP(\zZ)$. As before, $\Delta_t(\xi)$ and $I_t(\xi)$ extend linearly. At time $t$, \emph{contextual IDS} computes a distribution $\xi$ with marginal $\xi_z = \nu$, that minimizes the joint ratio,
\begin{align*}
\xi_t = \argmin_{\xi \in \sP(\xX \times \zZ),\,\xi_z = \nu} \Psi(\xi)\,.
\end{align*}
\looseness=-1 The action is sampled from $x_t \sim \xi_t(x | z_t)$ after observing $z_t$. In the joint minimization of the information ratio the contextual distribution $\nu$ contributes to exploration and a smaller information ratio. The intuition is that to estimate along a direction $x-y$ in a contextual action set $x,y \in \xX_z$, the learner can wait for a different context $z'$ to be realized where $x-y$ can easily be estimated and at low cost. This leads to the following condition that defines \emph{globally observable contextual games}:
\begin{align}
\forall z \in \zZ \text{ and } x,y \in \xX_z \Rightarrow \exists z' \in \zZ \text{ s.t. } x-y \in \laspan(A_x : x \in \xX_{z'}) \text{ .}\label{eq:context-globally}
\end{align}
The regret bound depends on the probability that a context occurs where estimation is possible. The \emph{expected worst-case alignment} $\alpha(\nu)$ is defined in Appendix \ref{app:contextual-ids}, Eq.\ \eqref{eq:exp-alignment}. It satisfies the intuitive upper bound $\alpha(\nu) \leq \EE_{\nu}[\alpha(z)]$ and recovers the previous definition for Dirac delta distributions (Lemma \ref{lem:contextual-aligment-bounds}, Appendix \ref{app:contextual-ids}). For finite games with finite context set, it further holds that
\begin{align*} 
\alpha(\nu) \leq  \max_{v\in \RR^d} \max_{z \in \zZ} \max_{x,y \in \xX_z} \min_{z' \in \zZ} \min_{u \in \xX_{z'}} \frac{\ip{v,x-y}^2}{\nu(z')\|A_{u}^{z'\T} v\|^2}\,,
\end{align*}
thus it suffices that a direction $x - y$ in $\xX_z$ can be estimated under {\em some} context $z'\in \zZ$ that appears with non-zero probability $\nu(z')$. The next result quantifies the rate in globally observable games. 
\begin{theorem}\label{thm:regret-contextual}
	For globally observable contextual games with bounded expected worst-case alignment $\alpha(\nu)$, for any $n \geq 1$, the regret is bounded with probability at least $1-\delta$,
	\begin{align*}
		R_n \leq C n^{2/3} \left(\alpha(\nu)\beta_n (\gamma_n + \log \tfrac{1}{\delta})\right)^{1/3} + 4\log\left(\frac{4n+4}{\delta}\right)\,.
	\end{align*}
\end{theorem}
All proofs and details for this result can be found in Appendix \ref{app:contextual-ids} and the analogous result for the locally observable case is in Appendix \ref{app:contextual-local}.

\section{Classification of Finite Games}

The upper bounds show that for globally observable games the regret is $\tilde O(n^{2/3})$, while for locally observable games it is $\tilde O(n^{1/2})$.
Of course, if there is only one Pareto optimal action, then the regret vanishes for any algorithm that just plays this action.
The classification theorem follows by proving that for games that are not globally observable, the regret is linear in the worst case, that for globally observable games
that are not locally observable the regret is $\Omega(n^{2/3})$ and that for locally observable games with more than one Pareto optimal action it is $\Omega(n^{1/2})$.
These lower bounds are supplied in \cref{sec:lower}. For simplicity, our results are for the \emph{expected minimax regret}, which is 
\begin{align*}
	R_n^* = \inf_{\pi} \sup_\theta \EE[R_n(\pi,\theta)] \,.
\end{align*}
The infimum is over policies $\pi = (\pi_t)_{t=1}^n$ defined by a sequence of $\fF_t$-measurable random variables on $\xX$ and $R_n(\pi,\theta) = \sum_{t=1}^n \ip{x^* - x_t,\theta}$ is the regret for parameter $\theta \in \RR^d$ when the actions are sampled from the policy $\pi$.

\begin{theorem}
The minimax regret for any finite linear partial monitoring game satisfies
\begin{align*}
R_n^* = 
\begin{cases}
0 & \text{if there is only one Pareto optimal action,} \\
\tilde \Theta(n^{1/2}) & \text{for locally observable games,} \\
\tilde \Theta(n^{2/3}) & \text{for globally observable games,} \\
\Omega(n) & \text{otherwise}\,.
\end{cases}
\end{align*}
\end{theorem}
The classification theorem is proven by combining upper and lower bounds, carefully checking that all cases have been covered. We further show in Appendix \ref{app:class} that our definitions of local and global observability coincide with the standard notions in finite partial monitoring that are based on the \emph{neighborhood graph}, as well as the notion of a \emph{global observer set} used by \cite{Lin2014combinatorial}.

\section{Applications and Extensions} \label{sec:applications}
The framework of linear partial monitoring captures many applications and models for sequential decision making that were previously studied in the literature. We outline some of them below and provide additional details in Appendix \ref{app:applications}.

\paragraph{Semi-Bandit and Full Information Feedback} The observation operators can be defined to yield {\em more} information than in the bandit case, up to revealing the parameter in each round ($A_x = \eye_{d}$). 
Naturally, additional information should only improve performance, but in our analysis, the bound degrades logarithmically with the observation dimension $m$. For the case of full information feedback, we show in Appendix \ref{app:full-info-setting} how to improve the bounds to get $R_n \leq \tilde \oO(\sqrt{dn})$.
\todojg{Network bandits didn't fit.}
	
\paragraph{Linear Bandits} The \emph{stochastic linear bandit} setting is a special case of our setup with $A_x =x$ \citep{abe1999associative,auer2002confidencebounds,dani2008stochastic,Abbasi2011improved}. 
Our analysis achieves the optimal $R_n \leq \tilde\oO(d\sqrt{n})$ dependency for the regret and generalizes the results for heteroscedastic bandits by \cite{Kirschner2018}. 
The UCB algorithm \citep{ACF02} has a distinct relation to the IDS framework, as we explain in Appendix \ref{app:ucb}.

\paragraph{Dueling Bandits} \looseness=-1 In \emph{dueling bandits}, the learner chooses a pair of actions and receives binary feedback indicating which action has higher reward \citep{yue2009interactively}. This feedback model can be cast as partial monitoring game \citep{gajane2015utility}. Let $\xX_0 \subset \RR^d$ be a ground set and $\xX = \xX_0 \times \xX_0$. The relative feedback for $(x_1,x_2) \in \xX$ is defined with $A_{x_1,x_2}=x_1-x_2$ and noise is added such that $\a_t \in \{-1,1\}$ is binary with expectation $A_{x_1,x_2}^\T\theta$. Note that bounded noise is subgaussian and our analysis applies. A possible reward model is to use averaged features $x_{x_1,x_2} = (x_1+x_2)/2$. Dueling bandits are locally observable if the learner can compare any pair of actions, and globally observable if comparisons are restricted to `adjacent' actions. See Appendix \ref{app:dueling} for details.

\paragraph{Combinatorial Bandits} This is the original motivation for the linear partial monitoring setting by \citet{Lin2014combinatorial} and \citet{Chaudhuri2016phased} and leads to games that are either locally or globally observable. We refer to the previous works for further applications. Our formulation covers combinatorial bandits both with bandit and semi-bandit feedback. An important special case is the batch setting (Appendix \ref{app:batch}). 

\paragraph{Transductive and Starved Bandits} The \emph{transductive linear bandit} setting was recently proposed by \cite{fiez2019transductive}. The learner has access to a set of actions that is dedicated for exploration, while the objective is to achieve low regret on a different, target set of actions. It was open to find an approach that minimizes cumulative regret, which we effectively resolve (Appendix \ref{app:transductive}). Similar in spirit are \emph{starved bandits} \citep{BCL17}, where the learner only obtains information when sampling actions from a pre-defined distribution. This setting is closely connected to our contextual setting (see Appendix \ref{app:starved}) and the regret bounds on convex action sets in Section \ref{sub:convex}.

\paragraph{Product Testing and Invasive Measurements} An early toy example for a globally, but not locally observable game is that of \emph{``apple tasting''} \citep{CBLuSt06}. In this task, the learner optimizes a production chain with the option to remove a product for inspection (and destroying it in the process). Other applications include parameter tuning of experimental facilities such as particle accelerators \citep{kirschner2019adaptive}, where invasive measurement devices provide a very rich signal at the expense of voiding any downstream measurements (for a stylized version of this problem and a numerical demonstration of IDS, see Appendix \ref{app:rkhs}).

\paragraph{Kernelized Partial Monitoring} Our approach and the analysis extend to the kernelized setting, where the reward function is in a known reproducing kernel Hilbert space (RKHS). This includes kernelized bandits \citep{Srinivas2009,AbbasiYadkori2012,Chowdhury2017}, also known as {\em Bayesian optimization}, as a special case. Interesting applications beyond the bandit setting include Bayesian optimization with gradients \citep{wu2017bayesian} or even Hessian evaluations \citep{wu2017exploiting}. Unlike previous results, our approach leverages all available information and achieves a strong finite time convergence guarantee. We refer the reader to Appendix \ref{app:rkhs} for a detailed introduction and formal statements. In the limit with continuous action sets, dueling bandits can be understood as \emph{global optimization} where the learner has access \emph{only} to the gradient.

\section{Discussion}\label{sec:discussion}
We introduced \emph{information directed sampling} for stochastic \emph{linear partial monitoring}, which -- to the best of our knowledge -- is the first approach that achieves the optimal regret rate in all finite linear games. Our classification theorem provides a complete picture of the achievable worst-case regret rates in finite linear games. Nevertheless, many directions are left for future work. Proving non-trivial instance-dependent regret bounds for IDS is an important open question, even for the standard linear bandit setting. Another challenge is to find precise observability conditions that capture the rate achievable on continuous action sets. 

For a naive implementation of IDS for finite games, the computational complexity per step is $\oO(d |\xX|^2)$, which is required to compute all gap estimates. The exact IDS distribution can be found by iterating over pairs of actions (a solution supported on two actions always exists). Alternatively a standard convex solver can be used to minimize the information ratio over the probability simplex. With a weaker regret estimate (Appendix \ref{app:regret-estimate}), the action minimizing the information ratio  can be found in $\oO(|\xX|)$, which matches the computational cost of index based approaches for bandits like UCB. For larger or continuous action sets, some previous approaches rely on oracle solvers \citep{Lin2014combinatorial,Chaudhuri2016phased} and for the bandit setting, Thompson sampling is a well-known oracle efficient method \citep{agrawal2013thompson,abeille2017linear}. Given the generality of our results, finding an oracle-efficient approximation of IDS is an important task for future work.

\newpage


\acks{This project has received funding from the European Research Council (ERC) under the European Union’s Horizon 2020 research and innovation programme grant agreement No 815943.}

\bibliography{references.bib}
\newpage

\appendix 

\section{Additional Lemmas and Proofs}

\subsection{Linear Partial Monitoring: General Setup}\label{app:general-setup}

Our setting can be formulated more generally, to allow applications where the learner can choose between different observation maps that are associated to the same action. Let $\iI$ be a compact index set. Each $i\in \iI$ indexes an action-observation tuple $(x_i, A_i)$ and the collection of such tuples represents a game $\gG = \{ (x_i, A_i) : i \in \iI \} \subset \RR^d \times \RR^{d\times m}$. At step $t$, the learner chooses an action index $i_t$ and observes the outcome $A_{i_t}^\T \theta + \epsilon_t$. The unobserved reward is $\ip{x_i, \theta}$.  We assume that the map $i \mapsto (x_i, A_i)$ is continuous to guarantees that the IDS distribution exists. The dimension $m$ of the observation can also depend on the action $i$ in general, but for simplicity, we set $m_i = m$. We overload\todoj{Does it make sense to use this formulation from the beginning?}

\subsection{Finite Partial Monitoring}\label{app:finite-pm} 

Unlike the standard finite and linear bandit frameworks, finite partial monitoring is not quite a special case of the linear setting.
On the one hand, our setting permits infinite observation (and action) spaces, which are not usually covered by existing results.
On the other hand, the assumptions of our setting mean the algorithm does not recover known bounds for algorithms in the finite unstructured setting.
The main reason is that we do not restrict $\theta$ except in terms of $\norm{\theta} \leq 1$, while in the finite setting the $\theta$ is effectively constrained
to the probability simplex.
Consider the following finite game, characterized by reward and signal matrices
\begin{align*}
\mathcal R = 
\begin{pmatrix}
1 & 1 \\
0 & 0  \\
\end{pmatrix}\,,
\qquad
\Sigma
= 
\begin{pmatrix}
0 & 0 \\
0 & 0 
\end{pmatrix}\,.
\end{align*}
The signal matrix is such that the learner observes no information.
Meanwhile, however, the rewards are such that the learner knows immediately that the first action is optimal, so in the
finite partial monitoring literature this game is trivial and good algorithms suffer zero regret.
Our algorithm, however, does not assume that $\theta$ lies in the probability simplex, and when $\theta = (-\sqrt{2}, -\sqrt{2})$, the second
action is clearly optimal. The different assumptions on $\theta$ mean that this game is now hopeless and algorithms consequentially suffer linear regret. \todojg{Perhaps this can be resolved by intersecting the confidence set with the constraint.}

\subsection{Proof of Lemma \ref{lem:ids-regret}}\label{app:proof-ids-regret}
\begin{proof}
	Using Freedman's inequality one can get the following concentration result on the regret \cite[Lemma 13]{Kirschner2018}. For any fixed $n$, with probability at least $1-\delta/2$,
	\begin{align*}
	R_n \leq \frac{5}{4}\sum_{t=1}^n \Delta_t(\mu_t) + 4 \log\left(\frac{4n +4}{\delta}\right) \,.
	\end{align*}
	The first sum is bounded by 
	\begin{align*}
	\sum_{t=1}^n \Delta_t(\mu_t)
	&= \sum_{t\in G_\eta} \Delta_t(\mu_t)  +  \sum_{t\in [n]\setminus G_\eta} \Delta_t(\mu_t) 
	\leq n\eta + \sqrt{\sum_{t\in [n]\setminus G_\eta} \Psi_t(\mu_t) \sum_{t=1}^n I_t(\mu_t)}\,.
	\end{align*}
	The inequality follows from the definition of $G_\eta$ and we use Cauchy-Schwarz to bound
	\begin{align*}
	\sum_{t \in [n] \setminus G_\eta} \Delta_t(\mu_t)
	&\leq \sum_{t \in [n] \setminus G_\eta} \sqrt{\Psi_t(\mu_t) I_t(\mu_t)} 
	\leq \sqrt{\sum_{t \in[n] \setminus G_\eta} \Psi_t(\mu_t) \sum_{t \in [n]} I_t(\mu_t)} \,.
	\end{align*}
	In the last step, we also used the non-negativity of $I_t(\mu_t)$. Finally, the sum over expected information gain $\sum_{t=1}^n I_t(\mu_t)$ is close to the realized information gain $\sum_{t=1}^n I_t(x_t)$ with high probability. This is made precise in Lemma~3 of \citet{Kirschner2018}, which shows that if $I_t(x) \leq 1$, then with probability at least $1-\delta/2$, for any $n\geq 1$,
	\begin{align*}
	\sum_{t=1}^n I_t(\mu_t) \leq 2 \sum_{t=1}^n I_t(x_t)  + 4 \log\left(\frac{1}{\delta}\right) + 8
	\end{align*}
	Note that our boundedness assumptions $\|A_x\|_2 \leq 1$ and the fact that $V_t^{-1} \preceq \mathbf{1}_d$ imply the required assumption $I_t(x) \leq 1$.	By definition $\gamma_n = \sum_{t=1}^n I_t(x_t)$. A union bound over the previous displays completes the proof.
\end{proof}

\subsection{Proof of Lemma \ref{lem:greedy}}\label{app:proof-greedy}
\begin{proof}
	By assumption, for any $p \in [0, 1)$ and any $x \in \xX$,
	\begin{align*}
	\Psi_t(\mu_t) 
	= \frac{\Delta_t(\mu_t)^2}{I_t(\mu_t)} 
	\leq \frac{\Delta_t((1-p) \mu_t + p \delta_x)^2}{(1 - p) I_t(\mu_t)}
	=: \Psi_t(p)\,
	\end{align*}
	Since $\Psi_t(0) = \Psi_t(\mu_t)$ and $p \mapsto \Psi_t(p)$ is differentiable at $p = 0$ it follows that 
	\begin{align*}
	0 \leq \Psi_t'(0) = \frac{2 \Delta_t(\mu_t) \Delta_t(x) - \Delta_t(\mu_t)^2}{I_t(\mu_t)}\,.
	\end{align*}
	The claim follows by rearranging.
\end{proof}

\subsection{Proof of Lemma \ref{lem:local-exp}}\label{app:proof-local-exp}
For the analysis it is useful to define a lower bound on the regret,
\begin{align}\label{eq:regret-lower}
\delta_{t}(x) &= \min_{\theta \in \cC_{t-1}} \max_{y \in \xX}  \ip{y - x, \theta}\,. 
\end{align}
By definition, with probability at least $1-\delta$ it holds that $\delta_t(x) \leq \ip{x^*-x,\theta}$. The set of plausible maximizers is equivalently described by $\pP_t = \{x \in \xX: \delta_t(x) = 0 \}$ and by continuity $\pP_t$ is a compact set. We further define the relaxed bound $\tilde \delta_t(x) = \max_{x' \in \xX} \min_{\theta \in \cC_{t-1}} \ip{x' - x, \theta}$. By the minimax inequality it holds that $\tilde \delta_t(x) \leq \delta_t(x)$. For $\tilde \delta_t$, we can explicitly solve the inner maximization to get $\tilde\delta_t(x) = \max_{y \in \xX} \ip{y- x,\hat{\theta}_{t-1}} - \beta_{t-1}^{1/2} \|y-x\|_{V_{t-1}^{-1}}$.

\begin{lemma}\label{lem:gap-maximisers}
	For the upper bound on the regret $\Delta_t(x)$, it holds that
	\begin{align*}
	\Delta_t(x) = \max_{y \in \pP_t} (y- x)^\T \hat{\theta}_{t-1} + \beta_{t-1}^{1/2} \|y - x\|_{V_{t-1}^{-1}}\,,
	\end{align*}
	where we restricted the maximum to plausible maximizers.
\end{lemma}
\begin{proof}[of Lemma \ref{lem:gap-maximisers}]
	Assume that $y$ is not a plausible maximizer, i.e. $\delta_t(y) > 0$. Then for any $\theta \in \cC_t$, there exists a $z = z(\theta) \in \xX$ s.t. $(z-y)^\T\theta > 0$. For fixed $x \in \xX$ we find,
	\begin{align}
	\max_{\theta \in \cC_t}(y - x)^\T\theta \label{eq: inside max} < \max_{\theta \in \cC_t}(z_\theta- x)^\T\theta \leq \max_{\theta \in \cC_t} \max_{z}(z - x)^\T \theta = \Delta_t(x) \,.
	\end{align}
	Hence, the left-hand side is maximized only if $y \in \pP_t$ is a plausible maximizer.
\end{proof}

\begin{proof}[of Lemma \ref{lem:local-exp}]
	Lemma \ref{lem:gap-maximisers} shows that we can write $\Delta_{t+1}(x)$ as follows:
	\begin{align*} 
	\Delta_{t}(x) = \max_{y \in \pP_t} \ip{y-x,\hat \theta_{t-1}} + \beta_{t-1}^{1/2} \|y-x \|_{V_{t-1}^{-1}} \,.
	\end{align*}
	Further, for any plausible action $x \in \pP_t$, we can bound the estimated gap by the associated uncertainty, $\ip{y-x,\hat \theta_{t-1}} \leq \beta_{t-1}^{1/2}\|y-x\|_{V_{t-1}^{-1}}$, which follows from the fact that $0 = \delta_t(x) \geq \tilde{\delta}_t(x)$.
	This implies that for all $x \in \pP_t$, $\Delta_t(x) \leq 2 \beta_{t-1}^{1/2}\max_{y \in \pP_t} \|x-y\|_{V_{t-1}^{-1}}$.  Specifically, let $w_t$ be the most uncertain direction in the set of plausible maximizers $w_t = \argmax_{w \in \{x-y : x,y \in \pP_t\}} \|w_t\|_{V_{t-1}^{-1}}^2$. Then, for $x \in \pP_t$,
	\begin{align*}
	\Delta_t(x)^2 \leq 4 \beta_{t-1}  \|w_t\|_{V_{t-1}^{-1}}^2 \leq 8 \alpha_t \beta_{t-1} \max_{z \in \pP_t} I_t(z) \,.
	\end{align*} 
	The last step follows from the same argument as in the proof of Lemma \ref{lem:global-exp}, where we restrict $x,y$ to $\pP_t$ and use the definition $\alpha_t = \alpha(\cC_{t-1})$.
\end{proof}

\subsection{Bounds for the Alignment Constant} \label{app:alignment-bounds}
\begin{lemma} \label{lem:alignment-bounds}
	Let $\pP \subset \xX$ finite with $p= |\pP|$ such that for all $x,y\in \pP$, $x-y \in \laspan(A_z : z \in \pP)$. Let $A \in \RR^{d \times pm}$ be the matrix formed by concatenating $(A_z : z \in \pP)$ and let $B$ be a subset of at most $d$ columns of $A$ such that $\laspan(B) = \laspan(A)$.
	Then 
	\begin{align*}
	\alpha(\pP) = \max_{v \in \RR^d} \max_{x,y \in \pP, x \neq y}  \frac{\ip{x-y, v}^2}{\max_{z \in \pP}\|A_z^\T v\|^2} \leq \min_{ w: Aw = x-y} \left(\sum_{z \in \pP} \|w_z\|\right)^2 \leq d \lambda_{\min}(BB^\T)^{-1} \,.
	\end{align*}
	Further, in the bandit game (where $A_x = x$), $\alpha(\pP) \leq 4$.
\end{lemma}

\begin{proof}
	Let $x, y \in \pP$ with $x \neq y$.
	By assumption, there exists a $w$ such that $x - y = A w$ with $w \neq 0$.
	Then,
	\begin{align*}
	\ip{x - y, v}^2 = \ip{Aw, v}^2  =\ip{w, A^\top v}^2 = \left(\sum_{z \in \pP}\ip{w_z, A_z^\top v}\right)^2\,,
	\end{align*}
	where we denote by $w_z \in \RR^m$ the weights corresponding to $A_z$. An application of Cauchy-Schwarz proves the first inequality, 
	\begin{align*}
	\alpha(\pP) \leq \frac{\ip{x - y, v}^2}{\max_{z \in \pP} \norm{A_z^\T v}^2} \leq \frac{\left(\sum_{z \in \pP} \|w_z\| \|A_z^\top v\|\right)^2}{\max_{z \in \pP}\norm{A_z^\T v}^2} \leq \left(\sum_{z \in \pP} \|w_z\|\right)^2\,.
	\end{align*}
	In the bandit game we can choose $w_x = 1$, $w_y=-1$ and $w_z = 0, \forall z \in \pP \setminus \{x,y\}$, hence $\alpha(\pP) \leq 4$. In general, we  can choose $\qQ \subset \pP$ with $|\qQ|=d$ s.t. $w_z = 0$ for $z \in \pP \setminus \qQ$. Therefore (we reuse the symbol $w$ in a different dimension),
	\begin{align*}
		\alpha(\pP) \leq \left(\sum_{z \in \qQ} \|w_z\|\right)^2 \leq d \|w\|^2\,.
	\end{align*}
	Denote $B = (A_z : z\in \qQ)$. The solution that minimizes the right-hand side is the ordinary least-squares solution $w^* = (B^\T B)^\dagger B^\T (x-y)$ where $^\dagger$ denotes the pseudo inverse. Therefore, using the properties of the pseudo inverse and  $\|x-y\| \leq 1$,
	\begin{align*}
	\alpha(\pP) \leq d \|w^*\|^2 \leq d \lambda_{\max}\left(B (B^\T B)^\dagger(B^\T B)^\dagger B^\T\right)  = d \lambda_{\max}\left( (BB^\T)^{-1}\right) = d \lambda_{\min}(BB^\T)^{-1} 	\,.
	\end{align*}
\end{proof}

\section{Regret Estimators and Information Gain Functions}\label{app:other-info}
\subsection{Regret estimate}\label{app:regret-estimate}
Our regret estimate $\Delta_{t}(x) = \max_{\theta \in \cC_{t-1}} \max_{y \in \xX} (y - x)^\T \theta$ is defined the tightest way for the given confidence bounds (up to truncation for bounded gaps). \todojg{Can we avoid truncation?}
An interesting fact is that $\Delta_{t}(x)$ is a convex function because the maximum is over convex functions. The estimate can be relaxed to
\begin{align*}
\tilde \Delta_{t}(x) = \max_{y \in \xX} \ip{y,\hat \theta_{t-1}} + \beta_{t-1}^{1/2} \|y\|_{V_{t-1}^{-1}} - \left( \ip{x, \hat \theta_{t-1}} -  \beta_{t-1}^{1/2} \|x\|_{V_{t-1}^{-1}}\right)\,.
\end{align*}
It holds that $\Delta_t(x) \leq \tilde \Delta_t(x)$. For $\tilde \Delta_t$, the maximum over $\xX$ is independent of $x$, which reduces the computational complexity to compute the regret estimate from $\oO(|\xX|^2)$ to $\oO(|\xX|)$. The estimate $\tilde{\Delta}_t$ relies on directly estimating the value of $x^\T \theta$ for all actions $x \in \xX$, which is not always possible in the general partial monitoring setting. The bandit game is an example where this is possible.

\subsection{Directed Information Gain}
Various ways of defining the information gain $I_t(x)$ are discussed in \citep{Kirschner2018}. The choice $I_t(x) = \log\det(\mathbf{1}_{m} + A_x^\T V_t^{-1} A_x)$ that we use in our main exposition is perhaps the most natural starting point, as it corresponds to the mutual information $\II(x,A_x\theta;\theta|\fF_{t-1})$ if we define a corresponding Gaussian prior and likelihood. We denote $V_t|A_x = V_t + A_xA_x^\T$. For a fixed $w \in \RR^d$ the \emph{directed information gain} is
\begin{align}\label{eq:info-gain-directed}
I_t(x;w) := \log\left({\|w\|_{V_{t-1}^{-1}}^2}\right) - \log\left( {\|w\|_{(V_{t-1}|A_x)^{-1}}^2}\right)\,.
\end{align}
The definition corresponds to the Shannon mutual information $\II(x,\a_t;\<w,\theta\>)$ which measures the Gaussian entropy reduction of $\theta$ projected onto the subspace spanned by $w$. The next lemma shows the information processing inequality $I_t(x;w) \leq I_t(x)$.

\begin{lemma}[Information processing inequality] \label{lem:info-processing}
	For all $w,x \in \RR^d$, $I_t(x;w) \leq I_t(x)$.
\end{lemma}\todojg{Equality Statement?}
\begin{proof}
	The proof is an exercises in linear algebra and makes use of the Sherman-Morrison formula and the matrix determinant lemma.
	\begin{align*}
	I_t(x;w) &= \log\left(\frac{\|w\|_{V_{t-1}^{-1}}^2}{\|w\|_{(V_{t-1}|A_x)^{-1}}^2}\right)\\
	&= -\log\left(1 - \frac{w^\T V_{t-1}^{-1}A_x(\mathbf{1}_m + A_x^\T V_{t-1}^{-1} A_x)^{-1}A_x^\T V_{t-1}^{-1} w}{\|w\|_{V_{t-1}^{-1}}^2}\right)\\
	&\leq \max_{v \in \RR^d : \|v\|_2 = 1}  -\log\left(v^\T v - v^\T V_{t-1}^{-1/2}A_x(\mathbf{1}_m + A_x^\T V_{t-1}^{-1} A_x)^{-1}A_x^\T V_{t-1}^{-1/2}v\right)\\
	&= \max_{v \in \RR^d: \|v\|_2 = 1}  -\log\left(v^\T \left(\mathbf{1}_d - V_{t-1}^{-1/2}A_x(\mathbf{1}_m + A_x^\T V_{t-1}^{-1} A_x)^{-1}A_x^\T V_{t-1}^{-1/2}\right)v\right)\\
	&=  \log\left(\lambda_{\max}\left(\left(\mathbf{1}_d - V_{t-1}^{-1/2}A_x(\mathbf{1}_m + A_x^\T V_{t-1}^{-1} A_x)^{-1}A_x^\T V_{t-1}^{-1/2}\right)^{-1}\right)\right)\,.\\
	\intertext{We first used Sherman-Morrison to compute $(V_{t-1}|A_x)^{-1}$ and then maximize over $v = \frac{V_{t-1}^{-1/2} w}{\|w\|_{V_{t-1}^{-1}}}$.}
	&\leq  \log\left(\det\left(\mathbf{1}_d - V_{t-1}^{-1/2}A_x(\mathbf{1}_m + A_x^\T V_{t-1}^{-1} A_x)^{-1}A_x^\T V_{t-1}^{-1/2}\right)^{-1}\right)\\
	&= \log\left(\det\left(\mathbf{1}_m + A_x^\T V_{t-1}^{-1} A_x\right)\det\left(\mathbf{1}_m + A_x^\T V_{t-1}^{-1} A_x -       A_x^\T V_{t-1}^{-1/2}V_{t-1}^{-1/2}A_x\right)^{-1}\right)\\
	&= \log \det \left(\mathbf{1}_m + A_x^\T V_{t-1}^{-1} A_x\right) = I_t(x)
	\end{align*}
	The inequality follows because all eigenvalues of the matrix inside the determinant are not smaller than 1, and then the generalized matrix determinant lemma to rewrite the expression.
\end{proof}

\begin{lemma}\label{lem:explore-direction}
	Let $\yY \subset \xX$ be a subset of actions and let $w=x-y$ for $x,y \in \yY$ such that $w \in \laspan(\{A_x : x \in \yY\})$. Then the most informative action in the set $\yY$ satisfies
	\begin{align*}
	\|w\|_{V_{t-1}^{-1}}^2 \leq 2 \alpha(\yY) \max_{z \in \yY} I_t(z;w)\,.
	\end{align*}
\end{lemma}

\begin{proof}
	First, note that $\frac{1}{2}(\mathbf{1}_d  + A_z^\T V_{t-1}^{-1} A_z) \preceq \mathbf{1}_d$ by our assumption that $\|A_z\|_2 \leq 1$, hence
	\begin{align*}
	\|A_z^\T V_{t-1}^{-1} w\|^2 \leq 2 w^\T V_{t-1}^{-1}A_z(\mathbf{1}_m + A_z^\T V_{t-1}^{-1} A_z)^{-1}A_z^\T V_{t-1}^{-1} w
	\end{align*}
	We further bound the following fraction:
	\begin{align*}
	\min_{z \in \yY} \frac{(w^\T V_{t-1}^{-1} w)^2}{ \|A_z^\T V_{t-1}^{-1}w \|^2} \leq \max_{x,y \in \yY} \max_{v \in \RR^d} \min_{z \in \yY} \frac{\ip{x-y,v}^2}{\|A_z^\T v\|^2 } = \alpha(\yY)\,.
	\end{align*}
	Since $x \leq -\log(1-x)$ for all $x \in [0,1]$,
	\begin{align*}
	\|w\|_{V_{t-1}^{-1}}^2 &= \min_{z \in \yY} \frac{(w^\T V_{t-1}^{-1} w)^2}{\|A_z^\T V_{t-1}^{-1}w \|^2} \max_{z \in \yY} \frac{\|A_z^\T V_{t-1}^{-1}w \|^2}{w^\T V_{t-1}^{-1} w}\\
	&\leq  2 \alpha(\yY) \max_{z \in \yY} \frac{w^\T V_{t-1}^{-1}A_z(\mathbf{1}_m + A_z^\T V_{t-1}^{-1} A_z)^{-1}A_z^\T V_{t-1}^{-1} w}{\|w\|_{V_{t-1}^{-1}}^2}\\
	&\leq - 2 \alpha(\yY) \max_{z \in \yY} \log\left(1 - \frac{w^\T V_{t-1}^{-1}A_z(\mathbf{1}_m + A_z^\T V_{t-1}^{-1} A_z)^{-1}A_z^\T V_{t-1}^{-1} w}{\|w\|_{V_{t-1}^{-1}}^2}\right)\\
	&=  2 \alpha(\yY) \max_{z \in \yY} \log\left(\frac{\|w\|_{V_{t-1}^{-1}}^2}{\|w\|_{(V_{t-1}|A_z)^{-1}}^2}\right) \\
	&=  2 \alpha(\yY) \max_{z \in \yY} I_t(z;w)
	\end{align*}
	This completes the proof.
\end{proof}
\noindent Define the most uncertain direction in the set of plausible maximisers,
\begin{align}
w_t = \argmax_{w \in \{w=x-y : x,y \in \pP_t\}} \|w\|_{V_{t-1}^{-1}}^2 \,.
\end{align}
Our next results extends the regret bounds to the variant of IDS that uses $I_t(x,w_t)$ as information function. Note that the information processing inequality (Lemma \ref{lem:info-processing}) implies that $\sum_{t=1}^n I_t(x_t;w_t) \leq \sum_{t=1}^n I_t(x_t)$, and therefore the bound in Lemma \ref{lem:total-info-gain} on the total information gain $\gamma_n$ continues to hold.
\begin{theorem}
 IDS, defined with the directed information gain $I_t(x;w_t)$, achieves for any $n \geq 1$, with probability at least $1-\delta$, $R_n \leq \oO\big(n^{2/3} (\alpha \beta_n (\gamma_n + \log \tfrac{1}{\delta}))^{1/3}\big)$ on globally observable games, and $R_n \leq \oO\left((\alpha_0 \beta_n (\gamma_n + \log \frac{1}{\delta})n)^{1/2}\right)$ on uniformly locally observable games.
\end{theorem}
\begin{proof}
	The proof is the same as for Theorem \ref{thm:global-upper} and Theorem \ref{thm:local-upper}, but uses the stronger inequality of Lemma \ref{lem:explore-direction} to bound $\max_{x,y \in \pP_t} \|x-y\|_{V_{t-1}^{-1}}^2 \leq \alpha(\pP_t) \max_{z \in \pP_t} I_t(x;w_t)$.
\end{proof}
Unlike for IDS defined with $I_t(x)$, the information gain $I_t(x;w_t)$ requires to compute the set of plausible maximizers $\pP_t = \{x \in \xX : \delta_t(x) = 0\}$. This can be done by computing $\delta_t(x) = \min_{\theta \in \cC_t} \max_{y \in \xX} \ip{y-x, \theta}$ for each $x \in \xX$. Note that the minimization over $\theta$ is on a convex function and therefore can be solved efficiently. \todoj{Explain how this can be solved efficiently.}

\subsection{Relation to the UCB algorithm}\label{app:ucb}
\citet{Kirschner2018} refer to the algorithm that chooses $x_t = \argmin_{x \in \xX} \Psi_t(\delta_x)$ as \emph{deterministic IDS}. Optimizing over a deterministic action choice is computationally cheaper and sufficient to obtain $\tilde{\oO}(\sqrt{n})$ regret on locally observable games as evident by Lemma \ref{lem:local-exp}. We draw a connection to the UCB algorithm. For $m=1$  and $\|A_x\|_{V_{t-1}^{-1}}^2 \ll 1$ we have
\begin{align*}
I_{t}(x) = \log (1 + \|A_x\|_{V_{t-1}^{-1}}^2) \approx \|A_x\|_{V_{t-1}^{-1}}^2\,.
\end{align*}
Define $\tilde I_t(x) = \|A_x\|_{V_{t-1}^{-1}}^2$ and $\tilde \Delta_t(x) = \max_y \ip{y,\hat \theta_{t-1}} + \beta_{t-1}^{1/2}\|y\|_{V_{t-1}^{-1}} - \left(\ip{x,\hat \theta_{t-1}} - \beta_{t-1}^{1/2} \|x\|_{V_{t-1}^{-1}}\right)$. The next lemma shows that in bandit games ($A_x = x$), deterministic IDS with $\tilde \Delta_t$ and $\tilde I_t$ as gap estimate and information gain, is equivalent to the UCB algorithm.
\begin{lemma}
	For a bandit game, let $x_t^\ucb = \argmax_{x \in \xX} \ip{x,\hat \theta_{t-1}} + \beta_{t-1}^{1/2} \|x\|_{V_{t-1}^{-1}}$ be the UCB action. Then,
	\begin{align*}
	x_t^\ucb \in	\argmin_{x \in \xX} \frac{\tilde \Delta_t(x)^2}{\tilde I_t(x) }\,.
	\end{align*}
\end{lemma}
\begin{proof}
	A related result appears in \cite[Lemma 2.1]{wang2016optimization}.	The information-ratio of the UCB action is
	\begin{align*}
	\frac{\tilde \Delta_t(x_t^\ucb)^2}{\tilde I_t(x_t^\ucb)} 
	= \frac{\left(\ip{x_t^\ucb,\hat \theta_{t-1}} + \beta_{t-1}^{1/2}\|x_t^\ucb\|_{V_{t-1}^{-1}} - \left(\ip{x_t^\ucb,\hat \theta_{t-1}} - \beta_{t-1}^{1/2} \|x_t^\ucb\|_{V_{t-1}^{-1}}\right)\right)^2}{\|x_t^\ucb\|_{V_{t-1}^{-1}}^2} 
	= 4\beta_{t-1}\,.
	\end{align*}
	Further, for any $x\in \xX$, $\max_y \ip{y,\hat \theta_{t-1}} + \beta_{t-1}^{1/2}\|y\|_{V_{t-1}^{-1}} \geq  \ip{x,\hat \theta_{t-1}} + \beta_{t-1}^{1/2}\|x\|_{V_{t-1}^{-1}}$, therefore
	\begin{align*}
	\frac{\tilde \Delta_t(x)^2}{\tilde I_t(x)} 
	\geq \frac{\left(\ip{x,\hat \theta_{t-1}} + \beta_{t-1}^{1/2}\|x\|_{V_{t-1}^{-1}} - \big(\ip{x,\hat \theta_{t-1}} - \beta_{t-1}^{1/2} \|x\|_{V_{t-1}^{-1}}\big)\right)^2}{\|x\|_{V_{t-1}^{-1}}^2} 
	= 4\beta_{t-1}\,.
	\end{align*}
	This shows that the UCB action minimizes the deterministic information ratio. 
\end{proof}
\section{Applications and Extensions}\label{app:applications}

We discuss applications and extensions. Note that we make use of the generalized setup (Appendix \ref{app:general-setup}) where necessary. In this case $\pP(\cC) \subset \iI$ is defined to contain indexes plausible actions.

\subsection{Full Information} \label{app:full-info-setting}
The {\em full information setting} is perhaps not the most interesting case to study in the stochastic setting, because IDS reduces to the naive algorithm that aggregates the information and always plays greedy. Nevertheless, we demonstrate that the regret bounds improve given the additional information. Two natural settings are $A_x = \eye_{d}$ and $A_x = X$ where $X = (x \in \xX)$ collects the actions as columns. In games where the information gain does not dependent on the action, IDS simply picks a regret minimizing action, $x_t = \argmin_{x \in \xX} \Delta_t(x)$.  We show that IDS achieves $R_n \leq \tilde \oO(\sqrt{dn})$, which improves a factor $\sqrt{d}$ compared to the bandit setting. For simplicity, let $A_x = \eye_{d}$ and therefore $V_t = (t+1)\eye_d$. The information gain is 
\begin{align*}
I_t(x) = \log \det( \eye_{d} + (t+1)^{-1}\eye_{d}) = d \log(1 + \tfrac{1}{t})\,.
\end{align*}
Hence $\gamma_n = d \log(n)$, but the ratio for the greedy action $\hat x_t^*$ is 
\begin{align*}
	\Psi_t(\hat x_t^*) \leq \frac{\beta_t \max_{x,y} \|x-y\|_{V_{t-1}^{-1}}^2}{d \log(1 + \tfrac{1}{t})} \approx \beta_t d^{-1}\,.
\end{align*}
Given that $\beta_n \approx d \log(n)$, this means the overall bound is $R_n \leq \tilde\oO(\sqrt{dn})$. The same holds true for the directed information gain $I_t(x;w_t)$. Interestingly, here the improvements stem from a reduced total information gain $\gamma_n \approx \log(n)$, and the ratio remains $\Psi_t(x_t^*) \approx \beta_t \approx  d\log(t)$.

\subsection{Transductive Bandits}\label{app:transductive}
In the {\em transductive bandit setting} \citep{fiez2019transductive} the learner has access to a set of informative actions $\sS \subset \RR^d$ for exploration and a set of actions $\vV \subset \RR^d$ that, when played, return reward.  The sets are allowed to overlap or be contained in the other. In the original formulation the objective is to minimize the simple regret of a final recommendation on the target set $\vV$ by choosing actions only from $\sS$. When the objective is to minimize cumulative regret, we can model this setting as a partial monitoring game by defining action-observation tuples  $\gG_1 =\{(0, x) : x \in \sS\setminus \vV \}$, $\gG_2 = \{(x, 0) : x \in \vV \setminus \sS \}$ and $\gG_3 = \{(x, x) : x \in \sS\cap \vV \}$, corresponding to informative actions with zero reward, actions that return reward but no information, and actions with the usual bandit information. The game is defined by $\gG= \gG_1 \cup \gG_2 \cup \gG_3$. Depending on the sets $\sS$ and $\vV$, the game can be either locally observable or globally observable (or even infeasible).

\subsection{Starved Bandits}\label{app:starved}
In the {\em starved bandit setting} \citep{BCL17} the learner only receives information if the action is sampled from a predefined distribution. Let $\xX_0$ be a ground set of actions that, when played, yield no information ($A_x = 0$). Denote by $\nu \in \sP(\xX)$ the distribution that the learner can use for exploration and $z_t \sim \nu$ is a sample from the distribution in round $t$. The starved bandit setting is closely related to the contextual partial monitoring game with $\gG = \gG_0 \cup \{ (z_t, A_{z_t} = z_t) \}$ added to the set of action-observation tuples. This game is globally observable if the distribution $\nu$ is sufficiently diverse such that the samples $(z_t)_{t=1}^n$ span the set of differences $\{x-y : x,y \in \xX\}$. Note that on a curved actions set, the rate can still be $\oO(\sqrt{n})$ as shown by \cite{BCL17} (also compare our results on curved action sets in Section \ref{sub:convex}).

\subsection{Batch Setting}\label{app:batch}
In the batch setting, the learner commits to choosing $B$ actions before observing the associated outcomes. This is important for applications where querying the objective for a number of actions in parallel is cheaper (or faster) than obtaining individual evaluations. This setting can be naturally formulated as a combinatorial partial monitoring game with semi-bandit feedback. Let $\xX_0 \subset \RR^d$ be a ground set of actions. The learner chooses a batch $(x_1, \dots, x_B) \in \xX_0^B$.  In the special case of a bandit feedback game, the reward is $\ip{x_1 + \dots + x_B, \theta}$ and the observation operator is $A_{x_1,\dots,x_B} = (x_1, \dots, x_B)$. With general feedback matrices, the batch game is
\begin{align*}
	\gG = \left\{ \left(\sum_{i=1}^B x_i, \big(A_{x_1}, \dots, A_{x_B}\big) \right)  :  (x_1, \dots, x_B)\in \xX_0^B \right\} \,.
\end{align*}
The bandit batch game is locally observable with $\alpha_0 \leq 4B^2$ (see Lemma \ref{lem:alignment-bounds}). The disadvantage of this formulation is, however, that the action space is exponentially large. Finding an efficient approximation of the IDS distribution is an interesting direction for future work.

\subsection{Dueling Bandits with Average Reward}\label{app:dueling} Let $\xX_0 \subset \RR^d$ be a ground set of actions. The dueling bandit with average reward is the following game with index set $\iI= \xX_0 \times \xX_0$:
\begin{align*}
\gG = \left\{ \left(\frac{x_1 + x_2}{2}, x_1 - x_2\right) : (x_1, x_2) \in \iI \right\}\,.
\end{align*}
In words, the learner can pick any pair of actions $x_1, x_2 \in \xX_0$, obtains the average reward $(x_1 + x_2)^\T \theta / 2$ and a noisy observation of the reward difference $(x_1 - x_2)^\T \theta$. Note that the learner can also choose $(x_1, x_1)$ with reward $x_1^\T \theta$ and no observation. Let $\pP(\cC)$ be a plausible set of actions. The first observation is that if $(x_1, x_2) \in \pP(\cC)$ then $(x_1,x_1) \in \pP(\cC)$ and $(x_2, x_2) \in \pP(\cC)$, because $(x_1 + x_2)/2 \in [x_1, x_2]$ lays on the line segment between $x_1$ and $x_2$. \todojg{Moreover, $(x_1,x_1), (x_2, x_2)$ and $(x_1, x_2)$ are all contained in the neighborhood $\nN_{(x_1,x_1), (x_2, x_2)}$ (TODO: neighborhod undefined at that point).}
Let $(x_1, x_2), (y_1, y_2) \in \pP(\cC)$ be two plausible actions. We can choose a path $(z_1, \dots, z_l)$ with $z_1 = x_1$, $z_l = y_1$ and $(z_i,z_{i+1}) \in \pP(\cC)$. Therefore we can write 
\begin{align*}
x_1 - y_1 = \sum_{i=1}^{l-1} x_i - x_{i+1} = \sum_{i=1}^{l-1} A_{(x_i, x_{i+1})}\,.
\end{align*}
The difference $x_2 - y_2$ can be written similarly, which shows that $\frac{x_1 + x_2}{2} -\frac{y_1 + y_2}{2} \in \laspan(A_i : i \in \pP(\cC))$. This shows that the game is locally observable. Turning to the local alignment constant 
\begin{align*}
\alpha(\pP) =\max_{v\in \RR^d} \max_{(x_1,x_2), (y_1, y_2) \in \pP} \min_{(z_1,z_2) \in \pP} \frac{\ip{(x_1 + x_2)/2 - (y_1 + y_2)/2, v}^2}{\|(z_1 - z_2)^\T v\|^2 }\,.
\end{align*}
Using Lemma \ref{lem:alignment-bounds} and the path construction above we can bound $\alpha \leq l$ or $\alpha \leq C d$.

\paragraph{Tightening the Alignment Constant} Define the sets
\begin{align*}
\qQ_1(\cC) &= \{i \in \iI : x_i \in \conv(x_j : j \in \pP(\cC)) \}\\
\qQ_2(\cC) &= \{ i  \in \iI: \Delta_\cC(x_i) \leq \max_{j \in \pP(\cC)} \Delta_\cC(x_j)\}
\end{align*}
with the regret estimate $\Delta_\cC(x) = \max_{\theta \in \cC} \max_{j \in \pP(\cC)} (x_j - x)^\T\theta$.
Note that $\Delta_\cC(x)$ is a convex function which implies that $\pP(\cC) \subset \qQ_1(\cC) \subset \qQ_2(\cC)$, but equality is not true in general. The observation is that in locally observable games, we can play actions in  $\qQ_2(\pP_t)$ without worsening the regret bound. Consequently, the local alignment constant can be tightened to
\begin{align*}
\bar \alpha(\cC)	= \max_{v\in \RR^d} \max_{i,j \in \pP(\cC)} \min_{k \in \qQ_2(\cC)} \frac{\ip{x_i-x_j, v}^2}{\|A_k^\T v\|^2 }\,.
\end{align*}
Clearly, $\bar \alpha(\cC) \leq \alpha(\cC)$ and all regret bounds hold true with $\alpha$ replaced by $\bar \alpha$. For the dueling bandit game with average reward, recall that $x_1,  y_1\in \pP(\cC)$ and therefore $(x_1 + y_1)/2 \in \conv(x_1, y_1)$, and the same holds true for $x_2,y_2$. This means we can now choose $A_{(x_1, y_1)} = x_1 -y_1$ and $A_{(x_2, y_2)}=x_2 - y_2$ as a response to estimate along the direction $(x_1 + x_2)/2 - (y_1 + y_2)/2$. We then write
\begin{align*}
\frac{x_1 + x_2}{2} - \frac{y_1 + y_2}{2} = \frac{1}{2} A_{(x_1, y_1)} + \frac{1}{2} A_{(x_2, y_2)}\,,
\end{align*}
and therefore, using the argument of Lemma \ref{lem:alignment-bounds}, $\bar \alpha(\cC) \leq 1$. 

\subsection{Partial Monitoring in Reproducing Kernel Hilbert Spaces}\label{app:rkhs}

The \emph{kernelized} setting is a practically relevant extension of the linear setting, where the feature dimension can be infinite. Let $\xX_0$ be a ground set of actions, not to be confused with the features. This is often a subset of $\RR^d$ but can be defined on other structures (e.g.\ graphs) as well. The actions $x \in \xX_0$ exhibit a non-linear dependence on the features through a positive-definite kernel map $k : \xX_0 \times \xX_0 \rightarrow \RR$. Let $\hH$ be the reproducing kernel Hilbert space (RKHS) corresponding to the given kernel $k$ and Hilbert norm $\|\cdot\|_{\hH}$. Vectors in $\hH$ represent functions over $\xX_0$, so we denote the unknown parameter by $f \in \hH$ (instead of $\theta$). The standard boundedness assumption is that the unknown function has bounded Hilbert norm $\|f\|_{\hH} \leq 1$. The kernel features $k_x = k(x, \cdot) \in \hH$ satisfy $f(x) = \ip{k_x, f}$ according to the reproducing property and the set of kernel features associated to the actions is $\xX = \{k_x : x \in \xX_0\}$. The best action is $x^* = \argmax_{x \in \xX_0} f(x)$, and the regret is 
\begin{align*}
	R_n = \sum_{t=1}^n f(x^*) - f(x_t)\,.
\end{align*}
The linear observation functions are linear operators $A_x : \hH \rightarrow \RR^m$. As before, the observations when choosing $x_t$ are $\a_t = A_{x_t}f + \epsilon_t$. The regularized kernel least squares estimator is
\begin{align}\label{eq:kernel-rls}
\hat{f}_t = \argmin_{f \in \hH} \sum_{s=1}^t \|A_{x_s} f - \a_s \|^2 + \|f\|_\hH^2\,.
\end{align}
In the bandit setting, the \emph{kernel trick} allows to express all quantities of interest in terms of the inner product $\ip{k_x,k_y} = k(x,y)$ evaluated on observed data points. In the general case where observations are generated from the observation operators $A_x$, we will need a slightly stronger assumption.
Denote the adjoint map of $A_x$ by $A_x^* : \RR^m \rightarrow \hH$. The requirement is that the matrix $M_{x,y} = A_x A_y^* \in \RR^{m\times m}$ and the vectors $k_x A_y^* \in \RR^m$ can be computed for any $x,y \in \xX_0$ (the theory also holds without the assumption, but it is needed to implement the algorithm if the feature dimension is infinite). We detail such a computation in examples below. By (a slight modification of) the representer theorem, we can write the solution to \eqref{eq:kernel-rls} as $\hat{f_t} = \sum_{s=1}^t A_{x_s}^* \varphi_s$ for weights $\varphi_s \in \RR^m$. Denote $\mathbf \a_t \in \RR^{mt}$ the vector that collects all observations $\mathbf \a_t = (\a_1^\T, \dots, \a_t^\T)^\T$, $K_t \in \RR^{mt\times mt}$ the kernel matrix that collects the matrices $(K_t)_{ij} = M_{x_i,y_i}$ and $k_t(x) \in \RR^{mt}$ the evaluation vector $k_t(x) = (k_x A_{x_1}^*, \dots, k_x A_{x_t}^*)^\T$. The solution $\hat{f}_t(x) = \ip{k_x, \hat f_t}$ to the least squares problem evaluated  at $x \in \xX_0$ is
\begin{align*}
	\hat{f}_t(x) = k_t(x)^\T (K_t + \eye_{mt})^{-1}\mathbf \a_t\,.
\end{align*}
The estimate corresponds to the posterior mean of a Gaussian process (GP) model with kernel $k$ and Gaussian likelihood \citep[c.f.][]{kanagawa2018gaussian}. The gap estimate at time $t+1$ is defined as
\begin{align*}
	\Delta_{t+1}(x) = \max_{y \in \xX_0} \hat f_{t}(y) - \hat f_{t}(x) + \beta_{t} \sqrt{\sigma_t(x)^2 + \sigma_t(y)^2 - 2k_t(x,y)}\,,
\end{align*}
where
\begin{align*}
	k_t(x,y) &= k(x,y) - k_t(x)^\T(K_t + \eye_{mt})^{-1}k_t(y)\,,\\
	\sigma_t(x) &= \sqrt{k_t(x,x)}\,,\\
	\beta_t^{1/2} &= \sqrt{\log \det(K_t + \eye_{mt})  + 2\log \tfrac{1}{\delta}} + 1\,.
\end{align*}
The estimate is chosen such that with probability at least $1-\delta$, $f(x^*) - f(x)\leq \Delta_t(x)$  for any $x \in \xX$ and $t \geq 1$ \citep[Theorem 3.11]{AbbasiYadkori2012}.

To compute the information gain, define $M_t(x) = (M_{x,x_1}^\T, \dots, M_{x,x_t}^\T)^\T \in \RR^{mt\times m}$. The kernelized information gain \eqref{eq:info-full} is given by
\begin{align*}
	I_t(x) = \log \det \left(\eye_m + M_{x,x} - M_t(x)^\T K_t^{-1} M_t(x)\right)\,.
\end{align*}

Denote by $k_{t|A_z}$ and $\sigma_{t|A_z}$ the uncertainty estimates that are (tentatively) updated with an observation generated from $A_z$. Such an update does not require the observation outcome $y_t$, similar to the linear case, where we can update the precision matrix $V_t|A_z = V_t + A_zA_z^\T$. Further, let $w_{x,y} = k_x - k_y$ be the difference of kernel features for the gap difference that we want to estimate. The kernelized directed information gain is
\begin{align*}
	I_t(z;w_{x,y}) = \log \left(\frac{\sigma_t(x)^2 + \sigma_t(y)^2 - 2 k_t(x,y)}{\sigma_{t|A_z}(x)^2 + \sigma_{t|A_z}(y)^2 - 2 k_{t|A_z}(x,y)}\right)\,.
\end{align*}
As before the information processing inequality (Lemma \ref{lem:info-processing}) implies that $I_t(z;w_{x,y}) \leq I_t(z)$. The bound in  Lemma \ref{lem:total-info-gain} on the total information gain $\gamma_n = \sum_{t=1}^n I_t(x_t) = \log \det(K_n + \eye_{mn})$ for finite feature dimension can be replaced by bounds that depend on the eigenspectrum of the kernel \citep{Srinivas2009}, for example $\gamma_n = \oO(\log(n)^{d+1})$ for the squared-exponential kernel on $\RR^d$. We remark that in the kernelized setting, only the computation of the estimator and information gain are different compared to the linear setting. The regret analysis remains the same with the appropriate constants $\beta_n$ and $\gamma_n$ , defined above. We therefore summarize our result:

\begin{corollary} The kernelized variant of IDS achieves $R_n \leq \oO\big(n^{2/3} (\alpha \beta_n (\gamma_n + \log \tfrac{1}{\delta}))^{1/3}\big)$ on globally observable games and  $R_n \leq \oO\left((\alpha_0 \beta_n (\gamma_n + \log\frac{1}{\delta}) n)^{1/2}\right)$ on uniformly locally observable games for any $n \geq 1$  with probability at least $1-\delta$.
\end{corollary}

\paragraph{Example: Kernelized Dueling Bandits}
We illustrate a dueling bandit setting, where the learner chooses two actions $(x,x') \in \xX_0^2$ and observes binary feedback on $f(x) \geq f(x')$. In the partial monitoring formulation, the observation operator is $A_{x,x'} = k_x - k_{x'}$, which means that the learner observes $\ip{k_x - k_{x'},f} = f(x) - f(x')$ up to noise. The learner obtains the reward of the first action (other reward models are possible), so the set of action-observation tuples is  
\begin{align*}
\gG = \{\left(k_x, A_{x,x'} = k_x - k_{x'}\right) : (x,x') \in \xX_0 \times \xX_0 \}\,.
\end{align*}
The noise on the observation $\a_t = f(x_t) - f(x_t') + \epsilon_t$ is such that $\a_t \in \{-1,1\}$ and $\EE[\a_t] = f(x_t) - f(x_t')$ (i.e.\ $\PP[\a_t=1] = 1 - \PP[\a_t = -1] = (1 + f(x) - f(x'))/2$). The quantities that are required to compute the estimator are
\begin{align*}
	M_{(x,x'),(z,z')} &= (k_x - k_{x'})(k_{z} - k_{z'})^* = k(x,z) - k(x,z') - k(x',z) + k(x',z')\,,\\
	k_x A_{(z,z')}^* &= k_x (k_z + k_z')^* = k(x,z) - k(x,z')\,.
\end{align*}
Kernelized dueling bandits have been studied in the literature \citep{gonzalez2017preferential,sui2017correlational,sui2018advancements} as well as extensions with multi-point comparisons \citep{sui2017multi}. Assuming that the learner can compare any pair of actions $(x,x')$, the setting is locally-observable by nature with $\alpha \leq 1$. Therefore, IDS achieves a $\tilde \oO(\sqrt{\beta_n \gamma_n n})$ regret bound. The same holds true for the deterministic variant that simply chooses the action which minimizes the information ratio.

\paragraph{Example: Bayesian Optimization with Gradients} While Bayesian optimization (or kernelized bandits) is typically phrased for the noisy, zero-order oracle, previous work also incorporates gradient information where it is available \citep{wu2017bayesian}. We illustrate a setting where the learner \emph{only} observes the gradient, which can be understood as a type of dueling bandit. Let $\xX_0 \subset \RR^d$ be a compact, connected domain and $f : \xX_0 \rightarrow \RR^d$ be an element in $\hH$ with a kernel that guarantees that any $f \in \hH$ is continuously differentiable. The gradient $\nabla_x$ acts linearly on the function $f$ and therefore is a valid choice for the observation operator $A_x : \hH \rightarrow \RR^d$ with $m=d$. The key step is to compute the quantities required for the estimation,
\begin{align*}
	(A_x A_y^*)_{ij} &= \ip{e_i, A_xA_y^* e_j} = \left(\nabla_x \ip{ A_y k_x,e_j}\right)_i = \frac{\partial}{\partial y_i}\frac{\partial}{\partial x_j} k(x,y)\,,\\
	(k_x A_y^*)_{i} &= \ip{k_x, A_y^* e_i} = \ip{\nabla_y k_x, e_i} = \frac{\partial}{\partial y_i} k(x,y) \,.
\end{align*}
The game where the learner observes only the gradient is globally observable, which means that for all $x,y \in \xX_0$, $k_x - k_y \in \laspan(A_x^* : x \in \xX_0)$. To see this, let $\tau : [0,1] \rightarrow \xX_0$ be a differentiable path with $\tau(0) = x$ and $\tau(1) = y$. We claim that
\begin{align*}
k_x - k_y = \int_0^1 A_{\alpha(t)}^*\dot \alpha(t) dt \,.
\end{align*}
This is verified, because for any $f \in \hH$ by the fundamental theorem of calculus,
\begin{align*}
 \left\langle\int_0^1 A_{\alpha(t)}^*\dot \alpha(t) dt,f\right\rangle &=  \int_0^1  \ip{A_{\alpha(t)}^*\dot \alpha(t), f} dt\\
&=  \int_0^1  \ip{\dot \alpha(t), A_x f} dt\\
&=  \int_0^1  \ip{\dot \alpha(t), \nabla_x f} dt\\
&= f(x) - f(y) = \ip{k_x - k_y, f}\,.
\end{align*}
If the learner observers both the function value and the gradient, the game is locally observable.

\paragraph{Example: Invasive Laser Alignment}
Consider a simplistic setup, where an experimenter wishes to align a laser on a squared target using two parameters $(x_1, x_2) \in \xX_0 = [-1,1]^2$ that correspond to a vertical and horizontal shift of the device (see Figure \ref{fig:laser} for an illustration). The power of the laser on a two-dimensional plane is given by an (initially) unknown function $f: \RR^2 \rightarrow \RR^2$. In the illustrated example it is set to $f(z_1,z_2) = \exp\left(- ((z_1 - 0.5)^2 + (z_2 - 0.5)^2)\right)$. The objective is to find a parameter setting that maximize the integrated intensity on the  (e.g.\ $1\times 1$) target, $I_t(x_1,x_2) = \int_{x_1}^{x_1+1} \int_{x_2 }^{x_2+1} f(z_1,z_2) dz_1dz_2$. At any step, the experimenter can choose to evaluate a setting $(x_1,x_2)$ and observes the corresponding intensity $I_t(x_1,x_2)$ up to noise. Since the intensity is the reward, this action has standard bandit feedback. Alternatively, the experimenter can drive a screen into the laser beam to measure the laser power on a $m\times m$-grid $G(x_1,x_2)$ centered at $(x_1,x_2)$, which yields $m^2$ noisy measurements $(f(z_1,z_2) : (z_1,z_2) \in G(x_1,x_2))$, possibly at a lower noise level than the integrated intensity measurement. As the screen blocks of the beam, there is no reward in such rounds (hence the term \emph{`invasive measurement`}). The learner therefore has the choice between a direct measurement of the objective and a more informative action that yields no reward. Clearly, the game is locally observable as each action contains the bandit feedback. We remark that the UCB algorithm \emph{never} chooses the invasive measurements, because the UCB score for these actions is always zero. On the other hand, IDS naturally trades of between the informative actions and those that lead to reward. In a (transductive) variant of the setup, the signal can \emph{only} be observed through the invasive measurements and the integrated signal is \emph{not} observed. In this case, the game is globally, but not locally observable.

We present a numerical simulation of this setup in Figure \ref{fig:laser}. Our set $\xX_0$ is discrete with 9 actions corresponding to a unit shift in any direction (or no shift). We use 25-dimensional features computed from a radial basis function kernel. In the setup where the reward signal can be observed directly, UCB outperforms IDS for the first $\sim 1000$ steps; but then IDS gains an advantage from choosing the more informative measurements from time to time. Without the direct reward observation, UCB continues to play actions that yield the integrated reward, but no longer receives any information. The parameter estimate is therefore never updated, and the UCB algorithm suffers linear regret. On the other hand, IDS still achieves no-regret through trading off the informative measurements with parameter settings that yield reward.
 
\begin{figure}[t]
	\centering
	\includegraphics{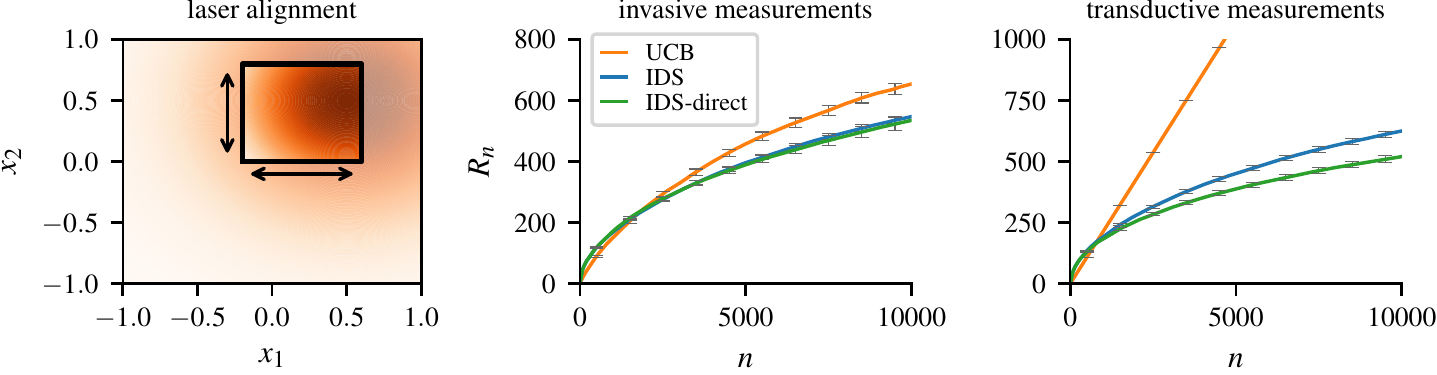}
	\caption{A demonstration of the stylized laser example. The left plot shows the energy  of the laser on the two dimensional plane. The objective is to shift the square target such that the integrated intensity within the square is maximized. The learner chooses to either observe a noisy measurement of the intensity, or alternatively, the energy function directly, evaluated on a measurement grid within the square (invasive measurements). The latter feedback is obtained from a screen that is put in the line of the laser, which blocks the beam and voids the reward signal. In the second variant (transductive measurements), the learner obtains information \emph{only} through the grid measurements. To solve the task, the learner needs to estimate the function and `blindly' move the target to the position with maximum integrated intensity. The plots on the right show the regret of IDS (with directed and undirected information gain) compared to the UCB algorithm. Note that UCB \emph{never} chooses the informative actions and therefore suffers linear regret on the second task. }\label{fig:laser}
\end{figure}

\section{Convex Action Sets}\label{app:convex}

The proof of Theorem~\ref{thm:convex} follows by using the curvature to bound the information ratio.
We will show the following:
\begin{align*}
\Psi_t(\mu_t) \leq C \beta_{t-1} \max\left(\frac{\diam(\xX)}{\kappa_{\circ}}, \diam(\xX)^2\right)\,.
\end{align*}
where $C > 0$ is a constant depending only on $(A_z : z \in \xX)$. Recall the definition of the support function $h_\xX(u) = \sup_{x \in \xX} \ip{x,u}$.  A simple calculation shows that $\nabla h_{\xX}(u) = \argmax_{x \in \xX} \ip{x, u}$, and $\hat{x}_t^* = \nabla h_{\xX}(\hat \theta_t)$ is the greedy action. Before the proof of the theorem we need a simple lemma bounding the regret in terms of the curvature.

\begin{lemma}\label{lem:curvature}
Suppose that $1/\kappa_{\circ} = \max_{x \in \xX} \lambda_{\max}(\nabla^2 h_{\xX}(x))$. Then
for any $\theta, \theta' \in \RR^d$,
\begin{align*}
\ip{\nabla h_{\xX}(\theta) - \nabla h_{\xX}(\theta'), \theta}
\leq \frac{2\snorm{\theta - \theta'}^2}{\kappa_{\circ} \snorm{\theta}}\,.
\end{align*}
\end{lemma}

\begin{proof}
Abbreviate $\eta = \theta / \norm{\theta}$ and $\eta' = \theta' / \norm{\theta'}$. Note that for $c > 0$, $h_\xX(c\cdot u) = c \cdot h_\xX(u)$, which implies that $\nabla h_\xX(c\cdot u) = \nabla h_\xX(u)$.
Using the definitions,
\begin{align*}
\ip{\nabla h_{\xX}(\theta) - \nabla h_{\xX}(\theta'), \theta}
&=\ip{\nabla h_{\xX}(\eta) - \nabla h_{\xX}(\eta'), \theta} \\
&= \norm{\theta} \ip{\nabla h_{\xX} (\eta) - \nabla h_{\xX}(\eta'), \eta} \\
&\stackrel{(i)}{\leq} \norm{\theta} \ip{\nabla h_{\xX}(\eta) - \nabla h_{\xX}(\eta'), \eta - \eta'} \\
&\stackrel{(ii)}{\leq} \norm{\theta} \frac{\snorm{\eta - \eta'}^2}{\kappa_{\circ} } \\
&\stackrel{(iii)}{\leq} \frac{2\snorm{\theta - \theta'}^2}{\kappa_{\circ} \norm{\theta}}\,,
\end{align*}
where inequality (i) follows because $\nabla h_{\xX}(\eta') = \argmax_{x \in \xX} \ip{x, \eta'}$.
The second inequality (ii) follows from the definition of $\kappa_{\circ}$ and because
\begin{align*}
\ip{\nabla h_{\xX}(\eta) - \nabla h_{\xX}(\eta'), \eta - \eta'} = \int^1_0 \snorm{\eta - \eta'}^2_{\nabla^2 h_{\xX}((1-t)\eta + t \eta')} dt 
&\leq \frac{\norm{\eta - \eta'}^2}{\kappa_{\circ}}\,.
\end{align*}
The last inequality (iii) follows from the following geometric inequality:
\begin{align*}
\forall x, y \in \RR^d\,\quad \norm{x - \tfrac{\norm{x}}{\norm{y}} y}^2 \leq 2 \norm{x - y}^2\,.
\end{align*}

\end{proof}

\begin{proof}\textbf{of Theorem~\ref{thm:convex}}\,\,
Let $\hat x_t^* = \nabla h_{\xX}(\hat \theta_{t-1})$. Then, by Lemma~\ref{lem:curvature},
\begin{align*}
\Delta_t(x_t^*) \leq \max_{y \in \xX, \phi \in \cC_{t-1}} \ip{y - x, \phi} \leq \max_{\phi \in \cC_{t-1}} \frac{2\snorm{\phi - \hat \theta_{t-1}}^2}{ \kappa_{\circ} \norm{\phi}}\,.
\end{align*}
Let $z_t = \argmax_{z \in \xX} I_t(z)$.
By the assumption that $\xX$ spans $\RR^d$ and the definition of global observability, it follows that $\laspan(A_z : z \in \xX) = \RR^d$, which means there exists a constant $c$ depending only
on $(A_z : z \in \xX)$ such that
\begin{align*}
\lambda_{\max}(V_{t-1}^{-1}) 
\leq c \lambda_{\max}(A_{z_t}^\top V_{t-1}^{-1} A_{z_t})
\leq c I_t(z_t)\,,
\end{align*}
where the second inequality follows from the same argument as in Lemma~\ref{lem:global-exp}. 
Hence,
\begin{align*}
\diam(\cC_{t-1})^2 
&= \max_{\theta, \phi \in \cC_{t-1}} \norm{\theta - \phi}^2 \\ 
&\leq \max_{\theta, \phi \in \cC_{t-1}} \lambda_{\max}(V_{t-1}^{-1}) \norm{\theta - \phi}_{V_{t-1}^{-1}}^2 \\
&\leq \lambda_{\max}(V_{t-1}^{-1}) \beta_{t-1} \\
&\leq c \beta_{t-1} I_t(z_t)\,.
\end{align*}
The analysis of the information ratio is decomposed into two cases. The first case is when $\cC_{t-1}$ has a large diameter, in which case the information ratio is well controlled
without using curvature, and by only exploration. 
Suppose that 
\begin{align}
2 \max\left(1, \sqrt{\frac{1}{\kappa_{\circ} \diam(\xX)}}\right) \diam(\cC_{t-1}) \geq \max_{\phi \in \cC_{t-1}} \norm{\phi}
\label{eq:convex:cond}
\end{align}
Then, using Cauchy--Schwarz inequality and the definition of $\Delta_t$,
\begin{align*}
\Delta_t(z_t)^2 
&\leq \diam(\xX)^2 \max_{\phi \in \cC_{t-1}} \norm{\phi}^2 \\ 
&\leq 4\max\left(1, \frac{1}{\kappa_{\circ} \diam(\xX)}\right) \diam(\xX)^2 \diam(\cC_{t-1})^2 \\
&\leq 4c \beta_{t-1} \max\left(1, \frac{1}{\kappa_{\circ} \diam(\xX)}\right) \diam(\xX)^2 I_t(z_t)\,,
\end{align*}
which implies that
\begin{align*}
\Psi_t(z_t) 
\leq 4 c \beta_{t-1}  \max\left(1, \frac{1}{\kappa_{\circ} \diam(\xX)}\right) \diam(\xX)^2\,.
\end{align*}
Moving to the second case where \cref{eq:convex:cond} does not hold.
Let
\begin{align*}
p = \frac{2 \diam(\cC_{t-1})^2}{\kappa_{\circ} \min_{\phi \in \cC_{t-1}} \norm{\phi} \max_{\phi \in \cC_{t-1}} \norm{\phi} \diam(\xX)}\,.
\end{align*}
That $p \in [0,1]$ follows by virtue of the fact that
\begin{align*}
\min_{\phi \in \cC_{t-1}} \norm{\phi} \geq \max_{\phi \in \cC_{t-1}} \norm{\phi} - \diam(\cC_{t-1}) \geq 
\frac{1}{2} \max_{\phi \in \cC_{t-1}} \norm{\phi}\,.
\end{align*}
Hence, $\mu = (1 - p) \delta_{\hat x^*_t} + p \delta_{z_t} \in \sP(\xX)$.
Using that $\hat x_t^* = \nabla h_{\xX}(\hat \theta_{t-1})$ and Lemma~\ref{lem:curvature},
\begin{align*}
\Delta_t(\mu) 
&\leq p \diam(\xX) \max_{\phi \in \cC_{t-1}} \norm{\phi} + \max_{\phi \in \cC_{t-1}} \ip{\nabla h_{\xX}(\phi) - \nabla h_{\xX}(\hat \theta_{t-1}), \phi} \\
&\leq p \diam(\xX) \max_{\phi \in \cC_{t-1}} \norm{\phi} + \frac{2}{\kappa_{\circ}} \max_{\phi \in \cC_{t-1}} \frac{\snorm{\hat \theta_{t-1} - \phi}^2}{\norm{\phi}} \\
&\leq p \diam(\xX) \max_{\phi \in \cC_{t-1}} \norm{\phi} + \frac{2 \diam(\cC_{t-1})^2}{\kappa_{\circ} \min_{\phi \in \cC_{t-1}} \norm{\phi}} \\
&= \frac{4 \diam(\cC_{t-1})^2}{\kappa_{\circ} \min_{\phi \in \cC_{t-1}} \norm{\phi}}\,.
\end{align*}
Therefore, using the fact that $c \beta_{t-1} I_t(\mu) \geq c \beta_{t-1} p I_t(z_t) \geq p \diam(\cC_{t-1})^2$,
\begin{align*}
\frac{\Delta_t(\mu)^2}{I_t(\mu)} 
&\leq \frac{16 c \beta_{t-1} \diam(\cC_{t-1})^2}{p\kappa_{\circ}^2 \min_{\phi \in \cC_{t-1}} \norm{\phi}^2} \\
&= \frac{8 c\beta_{t-1}\diam(\xX) \max_{\phi \in \cC_{t-1}} \norm{\phi}}{\kappa_{\circ} \min_{\phi \in \cC_{t-1}} \norm{\phi}} \\
&\leq \frac{16 c \beta_{t-1} \diam(\xX) }{\kappa_{\circ}}\,.
\end{align*}
Combining the two parts shows that
\begin{align*}
\Psi_t(\mu_t) \leq c \beta_{t-1} \max\left(\frac{16 \diam(\xX)}{\kappa_{\circ}},\, 4\max\left(1, \frac{1}{\kappa_{\circ} \diam(\xX)}\right) \diam(\xX)^2\right)\,.
\end{align*}
With the bound on the information ratio and Lemma~\ref{lem:ids-regret}, the proof of Theorem~\ref{thm:convex} follows now immediately.
\end{proof}

\section{Contextual Partial Monitoring}

\subsection{Conditional IDS for Contextual Games}\label{app:conditional-ids}

Conditional IDS optimizes the sampling distribution for the given context $z_t$,
\begin{align*}
	\mu_t = \argmin_{\mu} \Psi_t(\mu;z_t)\,.
\end{align*}
The computational complexity required to find the minimizer of the information ratio is the same as in the non-contextual case. We extend the notion of the alignment-constant with the contextual argument,
\begin{align}\label{eq:conditional-alignment}
	\alpha(\cC,z) = \max_{v \in \RR^d} \max_{x,y \in \pP(\cC,z)} \min_{u \in \pP(\cC,z)} \frac{\ip{x-y,v}^2}{\|A_{u}^{z\T} v\|^2}\,,
\end{align}
where $\pP(\cC,z) = \cup_{\theta \in \cC} \{x \in \xX_z: \ip{x,\theta} \geq \max_{y \in \xX_z}\ip{y,\theta}\}$ is the plausible maximizer set for context $z$. For globally observable games, we denote $\alpha(z) = \alpha(\RR^d,z)$.

The next result is an immediate upper bound for the regret of conditional IDS under the assumption that for any context $z \in \zZ$, each game $(\xX_z, \aA_z)$ is globally or locally observable, respectively.
\begin{corollary} \label{cor:conditional-ids}
	If a contextual game is globally observable in the sense that for any context $z$, the game $(\xX_z, \aA_z)$ is globally observable with uniformly bounded alignment constant $\alpha(z) \leq \alpha$, then for any $n \geq 1$ with probability at least $1-\delta$, conditional IDS achieves
	\begin{align*}
		R_n \leq C n^{2/3} \left(\alpha \beta_n (\gamma_n + \log\tfrac{1}{\delta})\right)^{1/3} + 4\log\left(\frac{4n+4}{\delta}\right)\,.
	\end{align*}
	If the contextual game is locally observable in the sense that for any $z \in \zZ$, the game $(\xX_z, \aA_z)$ is locally observable with uniformly bounded alignment constant $\alpha(\cC;z) \leq \alpha_0$ for all convex $\cC\subset \RR^d$, then for any $n \geq 1$ with probability at least $1-\delta$, conditional IDS achieves
	\begin{align*}
		R_n \leq C \sqrt{n \alpha_0 \beta_n \left( \gamma_n + \log \tfrac{1}{\delta} \right)} + 4\log\left(\frac{4n+4}{\delta}\right) \,,
	\end{align*}
	where $C$ is a universal constant. 
\end{corollary}
The corollary follows along the lines of our main results, Theorem \ref{thm:global-upper} \& \ref{thm:local-upper}. The assumptions of Corollary \ref{cor:conditional-ids} imply that the information ratio is bounded for any context $z$ in the respective regimes. One can achieve a slightly stronger result by replacing the alignment constant $\alpha_0$ with the average observed alignment $\frac{1}{n}\sum_{t=1}^n \alpha(\cC_t;z_t)$. In this case the bound explicitly depends on the sequence of observed contexts $(z_t)_{t=1}^n$ and the confidence sets $(\cC_t)_{t=1}^n$, which can lead to improved bounds in benign cases.

\subsection{Regret Bounds for Contextual IDS}\label{app:contextual-ids}
In this section we summarize results for \emph{contextual IDS}, which minimizes
\begin{align*}
\xi_t = \argmin_{\xi \in \sP(\xX \times \zZ),\,\xi_z = \nu} \Psi(\xi)\,,
\end{align*}
where $\nu \in \sP(\zZ)$ is a known distribution over the set of contexts. For general compact $\xX \times \zZ$, Prokhorov's theorem guarantees the existence of a minimizer \citep[cf.][]{Kirschner2018}. \todoj{Measurability of kernel?}

We overload the notation and let $\xX = \times_{z \in \zZ} \xX_z$ be the joint action space over all contexts and $\pP_t = \times_{z \in \zZ} \pP(\cC_t; z)$ the joined set of plausible maximisers. 
For a function $g : \zZ \rightarrow \zZ$ and a vector $x \in \xX$, we define $x_g \in \RR^{|\xX|}$ by $(x_g)_z = x_{g(z)}$. The \emph{expected alignment constant} is defined as
\begin{align}\label{eq:exp-alignment}
	\alpha(\cC,\nu) = \max_{v \in (\RR^d)^{\zZ}} \max_{x,y \in \pP(\cC)} \min_{u \in \pP(\cC)} \min_{g : \zZ \rightarrow \zZ} \EE_\nu\left[\ip{v,x-y}^{1/2}\right]^2 \EE_{\nu}\left[\frac{\|A_{u} v_{g} \|^2}{\ip{v_g,x_g - y_{g}}}\right]^{-1}\,.
\end{align}
As before, this corresponds to the signal-to-noise ratio that can be achieved by choosing the best aligned observation operator $\aA_u^z$ in context $z$, with the additional twist that learner can choose to estimate along a direction $x_{z'} -y_{z'}$ of a different context $z' = g(z)$. In the next lemma, we show that the definition satisfies more intuitive upper bounds. We will see that in the finite case, the definition of the alignment constant relates to natural conditions for local and global observability.

\begin{lemma}\label{lem:contextual-aligment-bounds} Let $\alpha(\cC,\nu)$ be the expected alignment \eqref{eq:exp-alignment} and $\alpha(\cC,z)$ the conditional alignment \eqref{eq:conditional-alignment}.
	\begin{enumerate}[i)] 
		\itemsep0pt
		\item For any convex $\cC \subset \RR^d$ and distribution $\nu \in \sP(\zZ)$ it holds that,
		\begin{align*}
		\alpha(\cC;\nu) \leq \EE_{z \sim \nu}[\alpha(\cC;z)]\,.
		\end{align*}
		\item For finite context sets,
		\begin{align*}
		\alpha(\cC,\nu) \leq \max_{z \in \zZ} \max_{x,y \in \pP(\cC;z), v\in \RR^d} \min_{z' \in \zZ} \min_{u \in \pP(\cC,z')} \frac{\ip{v,x-y}^2}{\nu(z')\|A_{u}^{z'\T} v\|^2} \,.
		\end{align*}
		\item For Dirac-delta distributions $\nu = \delta_{z}$,
		\begin{align*}
		\alpha(\cC,\delta_z) = \alpha(\cC,z) \,.
		\end{align*}
	\end{enumerate}

\end{lemma}
The first inequality implies that the information ratio of \emph{contextual IDS} is never worse than for \emph{conditional IDS}. The second inequality captures the intuition that for every direction $x - y \in \xX_z$ in a context $z$, there needs to be a context $z'$ that appears with positive probability $\nu(z') > 0$ where $x-y$ can be estimated. The last inequality is a sanity check which shows that for a constant context, we recover the previous definitions.

\begin{proof}
	For i), note that
	\begin{align*}
	\alpha(\cC;\nu) &= \max_{v \in (\RR^d)^{\zZ}} \max_{x,y \in \pP(\cC)} \min_{u \in \pP(\cC)} \min_{g : \zZ \rightarrow \zZ} \EE_\nu\left[\ip{v_z,x_z-y_z}^{1/2}\right]^2 \EE_{\nu}\left[\frac{\|A_{u_z}^z v_{g(z)} \|^2}{\ip{v_{g(z)},x_{g(z)} - y_{g(z)}}}\right]^{-1}\\
	&\leq \max_{v \in (\RR^d)^{\zZ}} \max_{x,y \in \pP(\cC)} \min_{u \in \pP(\cC)} \EE_\nu\left[\ip{v_z,x_z-y_z}^{1/2}\right]^2 \EE_{\nu}\left[\frac{\|A_{u_z}^z v_{z} \|^2}{\ip{v_{z},x_{z} - y_{z}}}\right]^{-1}\\
	&\leq \max_{v \in (\RR^d)^{\zZ}} \max_{x,y \in \pP(\cC)} \min_{u \in \pP(\cC)} \EE_\nu\left[ \frac{\ip{v_z,x_z-y_z}^2}{\|A_{u_z}^z v_{z} \|^2}\right]\\
	&= \EE_{\nu} \alpha(\cC;z)\,.
	\end{align*}
	The first inequality follows by choosing the identity function $g(z) = z$. The second inequality uses the fact that $(a,b) \mapsto a^2/b$ is convex on $\RR \times \RR_{>0}$ and Jensen's inequality.
	
	For ii), denote $z^* = z^*(v,x,y) = \argmax_{z \in \zZ} \ip{v_z,x_z-y_z}$ and define $g(z) = z^*$ for all $z \in \zZ$. Then
	\begin{align*}
	\alpha(\cC;\nu) &= \max_{v \in (\RR^d)^{\zZ}} \max_{x,y \in \pP(\cC)} \min_{u \in \pP(\cC)} \min_{g : \zZ \rightarrow \zZ} \EE_\nu\left[\ip{v_z,x_z-y_z}^{1/2}\right]^2 \EE_{\nu}\left[\frac{\|A_{u_z}^z v_{g(z)} \|^2}{\ip{v_{g(z)},x_{g(z)} - y_{g(z)}}}\right]^{-1}\\
	&\leq \max_{v \in (\RR^d)^{\zZ}} \max_{x,y \in \pP(\cC)} \min_{u \in \pP(\cC)} \ip{v_{z^*},x_{z^*}-y_{z^*}}^2 \EE_{\nu}\left[\|A_{u_z}^z v_{z^*} \|^2\right]^{-1}\\
	&\leq \max_{v \in (\RR^d)^{\zZ}} \max_{x,y \in \pP(\cC)} \min_{z' \in \zZ} \min_{u \in \pP(\cC;z')} \ip{v_{z^*},x_{z^*}-y_{z^*}}^2 \left( \nu(z')\|A_{u_{z'}}^{z'} v_{z^*} \|^2\right)^{-1}\,,
	\end{align*}
	which proves the claim. We first used the definition of $z^*$ and lower-bounded the expectation in the last step. The last equality iii) is immediate.
\end{proof}

Our next result extends Lemma \ref{lem:explore-direction} to account for the contextual distribution $\nu$ in the information ratio. We provide the proof for the tighter information gain $I_t(x,z;w)$ as defined in \eqref{eq:info-gain-directed}. The information processing inequality (Lemma \ref{lem:info-processing}) implies the result for $I_t(x,z)$.
\begin{lemma}\label{lem:explore-context}
	For a convex set $\cC\subset \RR^d$ let $w \in \{x - y : x,y \in \pP(\cC) \}$ be a difference in the plausible action set $\pP(\cC) = \times_{z \in \zZ} \pP(\cC;z)$. Then
	\begin{align*}
		\EE_{\nu}[\|w\|_{V_{t-1}^{-1}}]^2 \leq \alpha(\cC,\nu) \max_{u \in \pP(\cC)} \max_{g : \zZ \rightarrow \zZ} 2 I_t(u, \nu; w_g) \leq \alpha(\cC,\nu)  \max_{u \in \pP(\cC)} 2 I_t(u, \nu)\,.
	\end{align*}
\end{lemma}
\begin{proof}
	The proof is along the lines of Lemma \ref{lem:explore-direction}, but keeps the expectation over $\nu$. Let $g : \zZ \rightarrow \zZ$ be any function and $u \in \xX$. From the proof of the mentioned lemma, we find
	\begin{align*}
		\frac{\|A_{u_z}^{z\T} V_{t-1}^{-1}w_{g(z)}  \|^2}{\|w_{g(z)}\|_ {V_{t-1}^{-1}}} \leq 2 I_t(u,z;w_{g(z)})\,.
	\end{align*}
	Therefore, in expectation
	\begin{align*}
	\EE_{\nu}\left[\frac{\|A_{u}^\T V_{t-1}^{-1}w_{g}  \|^2}{\|w_{g}\|_ {V_{t-1}^{-1}}} \right]\leq 2 I_t(u,\nu;w_g)\,.
	\end{align*}
	Let $u^*,g^* = \argmax_{u \in \pP(\cC), g : \zZ \rightarrow \zZ}I_t(u, \nu; w_g)$. With this we find
	\begin{align*}
	\frac{\EE_{\nu}[\|w\|_{V_{t-1}^{-1}}]^2}{I_t(u^*, \nu; w_{g^*})} &\leq 2 \min_{u \in \pP(\cC)} \min_{g : \zZ \rightarrow \zZ}   \EE_{\nu}[\|w\|_{V_{t-1}^{-1}}]^2 \EE_{\nu}\left[\frac{\|A_{u}^\T V_{t-1}^{-1}w_{g}  \|^2}{\|w_{g}\|_ {V_{t-1}^{-1}}} \right]^{-1}\\
	&\leq  2 \max_{v \in \xX^{\zZ}} \max_{x,y \in \pP(\cC)} \min_{u \in \pP(\cC)} \min_{g : \zZ \rightarrow \zZ}  \EE_{\nu}[\ip{v,x-y}^{1/2}]^2 \EE_{\nu}\left[\frac{\|A_{u}^\T v_g  \|^2}{\ip{v_g, x_g-y_g}} \right]^{-1}\\
	&= 2 \alpha(\cC,\nu)
	\end{align*}
	Rearranging completes the proof.
\end{proof}
With these results the regret bounds for contextual IDS follow. For simplicity, the proof is given for IDS with full information gain \eqref{eq:info-full}, but similar results can be obtained for the directed information gain. First, the globally observable case (Theorem \ref{thm:regret-contextual}, Section \ref{sub:context}). 
\begin{proof}[\textbf{of Theorem \ref{thm:regret-contextual}}]
	Let $x^*_t \in \xX$ be the greedy action for each context, $x^*_t(z) = \argmax_{x \in \xX_z} x^\T \hat \theta_t$. Define the least accurate direction $w_t \in \{x - y : x,y \in \xX \}$ in the set $\pP_t$ as
	\begin{align} \label{eq:contextual-direction}
	w_t(z) = \argmax_{w = x-y : x,y \in \pP_t(z)} \|w\|_{V_{t-1}^{-1}}^2\,.
	\end{align}
	Recall that $x_t^*(z)^\T \hat \theta_t \leq \beta_{t-1}^{1/2}  \|w_t(z)\|_{V_{t-1}^{-1}}$. Lemma \ref{lem:explore-context} implies
	\begin{align*}
		\Delta_t(x_t^*;\nu)^2 &\leq \beta_{t-1} \EE_\nu[\|w_t\|_{V_{t-1}^{-1}}]^2 \leq 2 \beta_{t-1} \alpha(\nu) \max_{u \in \xX}I_t(u;\nu)\,.
	\end{align*}
	The rest of argument is analogous to the proof of Lemma \ref{lem:global-exp}. Consider a sampling distribution $\mu(p) \in \pP(\xX)$ that chooses $x^*_t(z)$ with probability $(1-p)$, and the most informative action $u(z) = \argmax_{u \in \xX_z} I_t(u,z)$ in context $z$ with probability $p$. By definition of the IDS policy,
	\begin{align*}
		\Psi_t(\xi_t) &\leq \min_{p \in [0,1]} \frac{\Delta_t(\mu(p);\nu)^2}{I_t(\mu(p);\nu)}\\
		&\leq \min_{p \in [0,1]} \frac{\left((1-p)\Delta_t(x_t^*;\nu) + p\right)^2}{p I_t(\mu(p);\nu)}\\
		&\leq 2\alpha \beta_{t-1} \min_{p \in [0,1]} \frac{\left((1-p)\Delta_t(x_t^*;\nu) + p\right)^2}{p \Delta_t(x_t^*;\nu)}\\
		&\leq  \frac{4\alpha \beta_{t-1}}{\Delta_t(x_t^*;\nu)^2}
	\end{align*}
	We first used that $I_t(x_t^*) \geq 0$, $\Delta_t(u) \leq 1$ and the inequality $2 \alpha \beta_{t-1} \max_u I_t(u;\nu)  \geq \Delta_t(x_t^*;\nu)^2$ that we derived above. Then we optimized over $p \in [0,1]$ in the last step.
	Similar to Lemma \ref{lem:greedy}, one can show that IDS plays greedy most of the time. For any $x \in \xX$,
	\begin{align*}
		\Delta_t(\mu_t;\nu) \leq 2 \Delta(x;\nu)\,.
	\end{align*}
	With  this we find
	\begin{align*}
	\Psi_t(\xi_t) \leq \frac{4\alpha \beta_{t-1}}{\Delta_t(x_t^*;\nu)}  \leq \frac{8\alpha \beta_{t-1}}{\Delta_t(\mu_t;\nu)} \,.
	\end{align*}
	Invoking the general bound (Lemma \ref{lem:ids-regret}) and balancing the terms completes the proof.	
\end{proof}

\subsection{Locally Observable Contextual Games}\label{app:contextual-local}
The condition for locally observable games is that any difference in the plausible action set $\pP(\cC,z)$ for a context $z \in \zZ$ can be estimated under possibly different context $z' \in \zZ$ by playing only actions that appear plausible optimal in the context $z'$. Formally,
\begin{align*}
\forall\; \cC \subset \RR^d \text{ convex,}\, z \in \zZ,\, x,y \in \pP_t(\cC;z) \Rightarrow \exists z' \in \zZ \text{ s.t. } x-y \in \laspan(A_x : x \in \pP_t(\cC;z')).
\end{align*}
If the condition holds true, Lemma \ref{lem:contextual-aligment-bounds} implies that $\alpha(\cC,\nu) \leq \oO\left((\min_{z \in \zZ} \nu(z))^{-1}\right)$ for finite action sets.

\begin{theorem}\label{thm:regret-contextual-local}
	In contextual games that are uniformly locally observable in the sense that $\alpha(\cC,\nu)\leq \alpha_0$ for any convex set $\cC\subset \RR^d$, the regret is bounded for any $n \geq 1$ with probability at least $1-\delta$,
	\begin{align*}
	R_n \leq C \sqrt{\alpha_0 \beta_n n (\gamma_n + \log \tfrac{1}{\delta})} + 4\log\left(\frac{4n+4}{\delta}\right)\,.
	\end{align*}
\end{theorem}
\begin{proof}
	Let $w_t$ be the least accurate direction in the current set of plausible maximisers defined in Eq.\ \eqref{eq:contextual-direction}. For any plausible maximiser $x \in \pP_t$ it holds that
	\begin{align*}
		\Delta_t(x,c) \leq 2 \beta_{t-1}^{1/2} \|w_t(c)\|_{V_{t-1}^{-1}}\,,
	\end{align*}
	Therefore $\Delta_t(x;\nu)^2 \leq 4 \beta_{t-1} \EE_{\nu}[\|w_t\|_{V_{t-1}^{-1}}]^2$. By Lemma \ref{lem:explore-context},
	\begin{align*}
		\Delta_t(x;\nu)^2 \leq 8 \beta_{t-1} \alpha(\cC_t;\nu) \max_{z \in \pP_t} I_t(z,\nu)\,.
	\end{align*}
	Finally, let $z_t = \argmax_{x \in \pP_t} I_t(z,\nu)$ be the most informative action that appears plausible optimal for each context. This action has bounded information ratio:
	\begin{align*}
		\Psi_t(\xi) \leq \frac{\Delta(z_t,\nu)^2}{I_t(z_t,\nu)} \leq 8 \alpha_0 \beta_{t-1} \,,
	\end{align*}
	where we also used that $\alpha(\cC_t,\nu) \leq \alpha_0$ by assumption. The result follows from Lemma \ref{lem:ids-regret}.
\end{proof}
\section{Proof of the Classification Theorem}\label{app:class}

The classification theorem is proven by combining upper and lower bounds, carefully checking that all cases have been covered.
To begin, we introduce the classification of actions that is now standard in finite partial monitoring games. 
The lower bounds then follow using standard techniques and are sketched in \cref{sec:lower}.

\paragraph{Notation} We denote by $\relint(C) = \{x \in C : \forall {y \in C} \; \exists {\lambda > 1}: \lambda x + (1-\lambda)y \in C\}$ the relative interior of a convex set $C$, and the set of extreme points by $\ext(C)$. The closure operator on subsets of a metric space is $\cl(\cdot)$ and $\dim(\cdot)$ is the Hausdorff dimension. For points $x, y \in \RR^d$, let $[x, y] = \{t x + (1 - t) y : t \in [0,1]\}$.

\begin{assumption}
For the remainder of this section we assume that $\xX$ is finite.
\end{assumption}

The set of \emph{Pareto optimal} actions is the set of extreme points $\ext(\conv(\xX))$ of the convex hull of $\xX$.
An action is \emph{degenerate} if it is on the boundary of $\conv(\xX)$, but not an extreme point.
Actions in the interior of $\conv(\xX)$ are called \emph{dominated}.
The situation is illustrated in Figure~\ref{fig:classification}. Finite partial monitoring games can be completely classified by considering a graph structure known as the \emph{neighborhood graph} \citep{lattimore2019cleaning}. Given an action $x \in \xX$, the \emph{cell} of $x$ is the set of parameters for which action $x$ is optimal: 
\begin{align*}
C_x = \{\theta \in \RR^d : x \in \pP(\theta)\} = \{\theta \in \RR^d : \max_{y \in \xX} \ip{y - x,\theta} = 0\}\,.
\end{align*}
Since $\xX$ is finite, $\conv(\xX)$ is a polytope and $C_x$ is either the singleton $\{0\}$ or an unbounded polytope.
An action $x$ is Pareto optimal if $\dim(C_x) = d$ and degenerate otherwise, which can be seen by observing that $C_x$ is the normal cone of $x$ with respect to the convex body $\conv(\xX)$. 
Pareto optimal actions $x$ and $y$ are called neighbours if $\dim(C_x \cap C_y) = d-1$, where the dimension of a polytope
is defined as the dimension of the smallest affine space containing it. 
The neighbourhood relation defines a connected graph on the set of Pareto optimal actions. 
For neighboring actions $x$ and $y$ let $\nN_{xy} = \{z : C_x \cap C_y \subseteq C_z\}$. Note that, besides $x$ and $y$, $\nN_{xy}$ contains only actions $z$ with $\dim(C_z) = d-1$. \citet{Lin2014combinatorial} and \citet{Chaudhuri2016phased} use a different notion to ensure global observability and to construct an explicit exploration distribution. A \emph{global observer set} is a set of actions $\yY \subset \xX$ such that $\laspan (A_x : x \in \yY) = \laspan(x-y : x,y \in \xX)$.

\begin{figure}
	\centering
	\includegraphics[width=6cm]{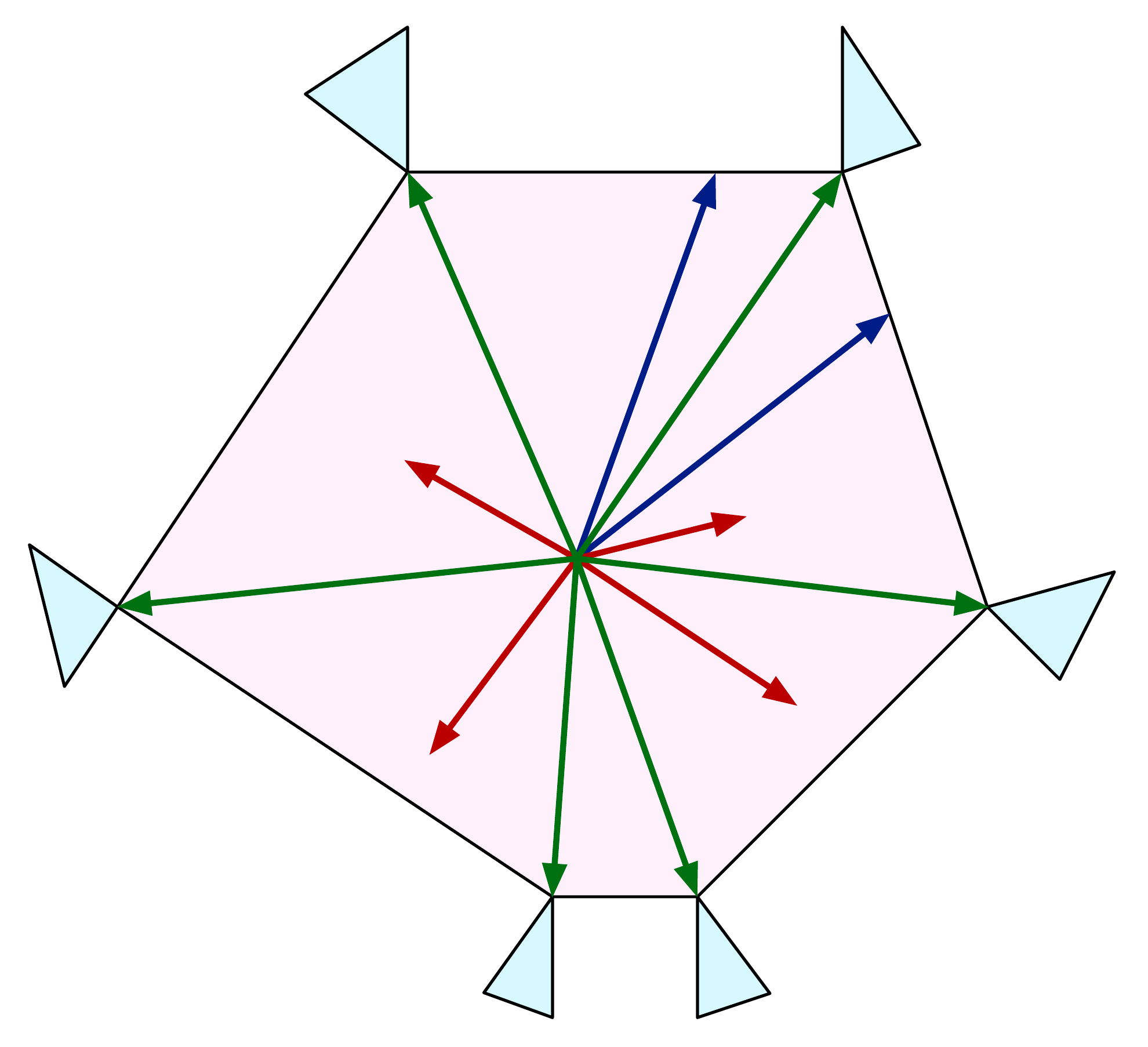}
	\caption{Green vectors are Pareto optimal, blue ones are degenerated and red are dominated. The light blue cones are the normal cones associated with each Pareto optimal action
		indicating the direction of $\theta$ for which that action is optimal.}\label{fig:classification}
\end{figure}

%


\begin{lemma}\label{lem:global-equivalence}
	The following conditions equivalently characterize globally observable games:
	\begin{enumerate}[i)]
		\itemsep0pt
		\item For all actions $x,y \in \xX$ it holds $x - y \in \laspan \{A_z : z \in \xX\}$.
		\item For all Pareto optimal actions $x,y$, it holds $x - y \in \laspan \{A_z : z \in \xX\}$.
		\item There exists a global observer set.
	\end{enumerate}
\end{lemma}
\begin{proof}
For the implication (ii $\Rightarrow$ i), note that Pareto optimal actions are the extreme points of $\conv(\xX)$, therefore any $x \in \xX$ can be written as a convex combination of Pareto optimal actions. (i $\Rightarrow$ iii) follows by taking $\xX$ as global observer set. (iii $\Rightarrow$ ii) immediately follows from the definition of a global observer set.
\end{proof}
The next lemma shows the relation of neighboring actions and local observability.
\begin{lemma}\label{lem:neighbors}
	Let $\xX$ be finite and $\cC \subset \RR^d$ be any convex set. Then
	\begin{enumerate}[i)]
		\itemsep0pt
		\item The Pareto optimal actions within $\pP(\cC)$ are connected on the neighborhood graph. 	
		\item For two Pareto optimal actions $x,y \in \pP(\cC)$ it holds that $\nN_{xy} \subset \pP(\cC)$.
		\item Any $x \in \pP(\cC)$ can be written as convex combination of Pareto optimal actions in $\pP(\cC)$.
	\end{enumerate}
\end{lemma}

\begin{proof}%
%
\begin{enumerate}[i)]
	\itemsep0pt
	\item The proof is intuitively simple. Take any Pareto optimal actions $x, y \in \pP(\cC)$ and let $\theta_x \in C_x$ and $\theta_y \in C_y$. Then take the chord $[\theta_x, \theta_y] \subset \cC$ and consider the path $(x_i)_{i=1}^n$ defined by the cells that intersect $[\theta_x, \theta_y] \cup C_{x_i} \neq \emptyset$. There is a technicality
	that this chord may pass through intersections of cells that have dimension $d - 2$. A perturbation and dimension argument fixes the proof (see Lemma \ref{lem:connect} below). For a similar result see \cite[Lemma 23]{lattimore2019information}.
	\item Let $x,y \in \pP(\cC)$ be Pareto optimal actions. Pick any $\theta \in C_x \cap C_y \cap \cC$. If $z \in \nN_{xy}$, we have $C_x \cap C_y \subset C_z$, hence $\theta \in C_z$ and $z$ is optimal for $\theta$. Therefore $z \in \pP(\cC)$.	
	\item Let $F$ be the lowest dimensional face of $\conv(\xX)$ containing $x \in \pP(\cC)$. Assume that $x$ is in the interior of $F$ (otherwise it would be an extreme point and so Pareto optimal).  Then let $\theta$ be a parameter such that $x$ is optimal. $H = \{z : (x - z)^\T  \theta = 0 \}$ is a supporting hyperplane of $\conv(F)$. Hence $F$ is a subset of $H$. Note that $H \cap \xX \subset \pP(\theta)$ contains actions that are optimal for $\theta$. Therefore all extreme points of $F$ are in $\pP(\cC)$ and since $x$ is in the convex hull of the extreme points of $F$ the result follows.
\end{enumerate}%
\end{proof}

The next lemma shows that observability can be characterized in terms of the neighborhood relation. 
\begin{lemma}\label{lem:local-equivalence}
	The following conditions equivalently characterize locally observable games:
	\begin{enumerate}[i)]
		\item For any convex $\cC$ and all $x,y \in \pP(\cC)$, $x-y \in \laspan\{A_z : z \in \pP(\cC)\}$.
		\item For any two neighboring Pareto optimal actions $x,y$, $x - y \in \laspan\{A_z : z \in \nN_{xy}\}$.
	\end{enumerate}
\end{lemma}
\begin{proof}
	``i) $\Rightarrow$ ii)''. Let $x,y \in \xX$ be neighboring Pareto optimal actions. Pick $\theta \in \relint(C_x \cap C_y)$. Then $\nN_{xy} = \pP(\theta)$ by Lemma \ref{lem:neighbors} and therefore $x - y \in \laspan\{A_x : x \in \pP(\cC)\}$ by i). 
	 
	 ``ii) $\Rightarrow$ i)''. Let $x,y \in \pP(\cC)$. First note that by Lemma \ref{lem:neighbors}, iii), $x-y$ can be written as linear combination of Pareto optimal actions in $\pP(\cC)$. Therefore we can assume that $x,y$ are Pareto optimal. By Lemma \ref{lem:neighbors}, i), there exists a sequence $(x_i)_{i=1}^m$ of Pareto optimal actions  with $x_1=x$ and $x_m = y$, such that $x_i,x_{i+1}$ are neighbors and  $\{x_i : i=1,\dots,m\} \subset \pP(\cC)$. By assumption, $x_i - x_{i-1} \in \laspan\{A_z : z \in \nN_{x_ix_{i-1}}\}$. Since $x-y = \sum_{i=1}^{m-1} x_i - x_{i+1}$ the claim follows.
\end{proof}

Lastly, the key lemma for proving the lower bound for globally observable games shows that in games that are not locally observable, there exists a pair of neighbouring Pareto optimal actions $x,y$ and a parameter $\theta$ such that both actions $x,y$ are optimal, but $\ip{x-y, \theta}$ can \emph{not} be estimated by playing only actions from the neighborhood $\nN_{xy}$.

\begin{lemma}\label{lem:local}
Suppose a game is not locally observable. Then there exists a pair $x, y$ of neighbouring Pareto optimal actions and $\theta \in \relint(C_x \cap C_y)$ such that $x - y \notin \laspan\{A_z : z \in \nN_{xy}\}$.
\end{lemma}

\begin{proof}
The lemma follows from the definition of local observability its equivalent charaterization provided in Lemma \ref{lem:local-equivalence}.
%
\end{proof}

\begin{lemma}\label{lem:connect}
Let $\uU$ be a collection of disjoint open sets of $\RR^d$ with the usual metric $d(x, y) = \norm{x - y}$ such that:
\begin{enumerate}
\item The union $\kK = \bigcup_{i=1}^\infty \cl(U_i)$ is convex.
\item $\dim\left(\kK \cap \{x : d(x, y) \leq \epsilon\}\right) = \dim(\kK)$ for all $\epsilon > 0$ and $y \in U \in \uU$.
\item $A = \bigcup_{U, V, W \in \uU \text{ are distinct}} (\cl(U) \cap \cl(V) \cap \cl(W))$ has $\dim(A) < \dim(\kK) - 1$.
\item For any compact set $W \subset \RR^d$, at most finitely many elements of $\uU$ have non-empty intersection with $W$.
\end{enumerate}
We say that $U_1, U_2 \in \uU $ are connected if $(\cl(U_1) \cap \cl(U_2)) \setminus A \neq \emptyset$.
Then, for any $x \in \cl(U)$ and $y \in \cl(V)$, there exists a sequence $(U_i)_{i=1}^m$ of connected sets with $U_1 = U$ and $U_m = V$ and $U_i \cap [x, y] \neq \emptyset$.
\end{lemma}

\begin{proof}
Suppose that $x \in \cl(U)$ and $y \in V$.
Let $B = \{z \in \kK : d(z, y) < \epsilon\}$ with $\epsilon$ sufficiently small that $B \subset V$ and $S = \{z \in B: [z, x] \cap A \neq \emptyset\}$.
A straightforward calculation shows that $\dim(S) \leq \dim(A) + 1 < \dim(\kK) = \dim(B)$.
Hence, there exists a $z \in B \setminus S$ and $[x,z] \cap A = \emptyset$.
By the last assumption, $[x, z]$ intersects with at most finitely many elements of $\uU$, which form the path between $U$ and $V$.
Next, suppose that $y \in \cl(V)$ and let $(y_n)$ be a sequence in $V$ with $\lim_{n\to\infty} d(y_n, y) = 0$.
By the previous argument, for each $n$, there exists a sequence $(U_i^n)_{i=1}^{m_n}$ of connected sets with $U^n_1 = U$ and $U^n_{m_n} = V$.
By the fourth assumption, for suitably large $n$, there are only finitely many sets in all the $(U^n_i)_{i,n}$ and hence, by re-labelling if necessary, 
the sequence of connected sets can be chosen so that $(U^n_i)_{i=1}^{m_n}$ converges (in the sense that the identity/order of the sequences converges -- the discrete topology on finite sequences of $\uU$) 
to some sequence $(U_i)_{i=1}^m$. We need to show that $\cl(U_i) \cap [x, y] \neq \emptyset$ for each $i$.
By the definition of convergence we have $[x, y_n] \cap \cl(U_i) \neq \emptyset$ for all suitably large $n$.
Taking a sequence $(z_n)$ with $z_n \in U_i$ for all suitably large $n$. Compactness again allows us to assume that $(z_n)$ converges to some $z$, which is easily seen to lie on $[x, z]$ and by closure of $U_i$ also
lies in $U_i$, as required.
\end{proof}

\section{Lower Bounds}\label{sec:lower}

\newcommand{\bbP}{\mathbb{P}}

The lower bounds complete the classification theorem.
These results are almost implied by existing theorems from finite partial monitoring. The only difference is that here the outcome space is infinite, which does not change
the structure of the proofs. We include here the key details and intuition.
As expected, the key tool is Le Cam's method in combination with the Bretagnolle--Huber inequality \citep{bretagnolle1979estimation} and an elementary calculation of the relative entropy between measures on interaction sequences induced by
a fixed policy and for different environments. For the remainder of this section, we fix an arbitrary policy and finite game with actions $\xX$ and feedback functions $(A_x)_{x \in \xX}$.
For simplicity, we assume the noise is Gaussian and $\xX$ spans $\RR^d$.
Given a $\theta \in \RR^d$ let $\bbP^n_\theta$ be the measure on action/observation sequences of length $n$ when the learner interacts with the game for parameter $\theta$.
Before the theorems and proofs we need a little more notation.
Let
\begin{align*}
V_n(\theta) = \EE_\theta[V_n] = \EE_\theta\left[\sum_{t=1}^n A_{x_t} A_{x_t}^\top\right]\,.
\end{align*}
Then define $E_n(\theta)$ as the binary random variable that the algorithm plays a suboptimal action at least $n/2$ times.
\begin{align*}
E_n(\theta) = \chf\left(\sum_{t=1}^n \chf(x_t \notin \pP(\theta)) \geq n/2\right)\,.
\end{align*}
Notice that if $\theta, \theta'$ are such that $\pP(\theta) \cap \pP(\theta') = \emptyset$, then $E_n(\theta') \geq 1 - E_n(\theta)$.
For simplicity we focus on proving lower bounds on the expected regret. The extension to high probability bounds is possible using the techniques of \cite{GL16}. Let
\begin{align*}
R_n(\theta) = \EE_\theta[R_n]
\end{align*}
be the expected regret when the learner interacts with the environment determined by $\theta$.

\begin{lemma}\label{lem:kl}
The relative entropy between $\bbP_\theta^n$ and $\bbP_{\theta'}^n$ satisfies $\displaystyle \KL(\bbP_\theta^n, \bbP_{\theta'}^n) = \frac{1}{2} \norm{\theta - \theta'}_{V_n(\theta)}^2$.
\end{lemma}
For a proof refer to \citep[Theorem 24.1]{lattimore2018bandit}.

\begin{lemma}\label{lem:bretagnolle-huber}(Bretagnolle-Huber inequality) Let $P$ and $Q$ be probability measures on the same measurable space $(F, \Omega)$ and let $A \in F$ be an arbitrary event. Then
	\begin{align}
		P(A) + Q(A^c) \geq \frac{1}{2} \exp(-\KL(P,Q))\,.
	\end{align}
\end{lemma}

\begin{theorem}\label{thm:linear}
Suppose that $\laspan(A_x : x \in \xX) \neq \RR^d$, then there exists a game-dependent constant $c > 0$ such that for all $n \geq 1$ there exists a $\theta$ for which
$R_n(\theta) \geq cn$.
\end{theorem}

\begin{proof}
Let $\theta \in \RR^d$ be a non-zero vector such that $A_x \theta = 0$ for all $x \in \xX$, which exists by the assumption that $\laspan(A_x : x \in \xX) \neq \RR^d$. 
Next, let $\theta' = -\theta$ and notice that by Lemma \ref{lem:kl},
\begin{align*}
\KL(\bbP_\theta^n, \bbP_{\theta'}^n) = 0\,.
\end{align*}
By our choice, the optimal action for the environment determined by $\theta$ and $\theta'$ are different: $\pP(\theta) \cap \pP(\theta') = \emptyset$.
The Bretagnolle-Huber inequality (Lemma \ref{lem:bretagnolle-huber}) implies that
\begin{align}
\bbP_\theta^n(E_n(\theta)) + \bbP_{\theta'}^n(E_n(\theta')) 
\geq \bbP_\theta^n(E_n(\theta)) + \bbP_{\theta'}^n(1- E_n(\theta)) 
\geq \frac{1}{2} \exp\left(-\KL(\bbP_\theta^n, \bbP_{\theta'}^n)\right)
\geq \frac{1}{2}\,. 
\label{eq:bh}
\end{align}
Furthermore, there exists an $\epsilon > 0$ such that $R_n(\theta) \geq \epsilon n \bbP_\theta^n(E_n(\theta)) / 2$.
Hence, by \cref{eq:bh}, the regret is linear for either environment $\theta$ or $\theta'$. 
\end{proof}

\begin{theorem}
Suppose the game is globally observable, but not locally observable. Then there exists a game-dependent constant $c > 0$ and  $\theta \in \RR^d$ such that the regret is $R_n(\theta) \geq c n^{2/3}$.
\end{theorem}

\begin{proof}
By Lemma \ref{lem:local}, there exists a pair of neighboring Pareto optimal actions $x, y \in \ext(\conv(\xX))$ and $\theta \in \relint(C_x \cap C_y)$ such that 
$x - y \notin \laspan(\{A_z : z \in \pP(\theta)\}) = L$.
Let $x - y = u + v$, where $u \in L$ and $v \in L^\perp$.
Since $x - y \notin L$ it follows that 
\begin{align*}
\ip{x - y, v} = \ip{u + v, v} = \norm{v}^2 > 0\,.
\end{align*}
In particular, for suitably small $\epsilon > 0$ it holds that
$\theta + \epsilon v \in C_x$ and $\theta - \epsilon v \in C_y$.
Define
\begin{align*}
\theta_n = \theta + n^{-1/3} v \quad \text{and} \quad
\theta_n' = \theta - n^{-1/3} v
\end{align*}
and let assume $n$ is sufficiently large that $\theta_n \in C_x$ and $\theta'_n \in C_y$.
Next, decompose $V_n(\theta)$ as $V_n(\theta) = U_n(\theta) + W_n(\theta)$, where
\begin{align*}
U_n(\theta) = \EE_{\theta}\left[\sum_{t=1}^n \chf(x_t \in \pP(\theta)) A_{x_t} A_{x_t}^\top\right] \quad \text{and} \quad
W_n(\theta) = \EE_{\theta}\left[\sum_{t=1}^n \chf(x_t \notin \pP(\theta)) A_{x_t} A_{x_t}^\top\right]\,.
\end{align*}
Let $T_n(\yY) = \sum_{t=1}^n \chf (x_t \in \yY)$ be the number of times an action in $\yY \subset \xX$ is played. Notice, since $v \in L^\perp$, that
\begin{align*}
\frac{1}{2}\norm{\theta_n - \theta'_n}^2_{V_n(\theta_n)}
&= 2n^{-2/3}\norm{v}^2_{V_n(\theta_n)}
= 2n^{-2/3} \norm{v}^2_{W_n(\theta_n)}
\leq 2n^{-2/3}\EE_{\theta_n}[T_n(\pP(\theta)^c)] \norm{v}^2_G\,,
\end{align*}
where $G = \sum_{x \in \xX} A_x A_x^\top$.
Now, there exists a game-dependent constant $\epsilon > 0$ such that
\begin{align*}
R_n(\theta_n) \geq \epsilon \EE_{\theta_n}[T_n(\pP(\theta)^c)]\,.
\end{align*}
Hence if $\EE_{\theta_n}[T_n(\pP(\theta)^c)] \geq n^{2/3}$, then $R_n(\theta_n) \geq \epsilon n^{2/3}$. Assume that $\EE_{\theta_n}[T_n(\pP(\theta)^c)] \leq n^{2/3}$. By the Bretagnolle-Huber inequality (Lemma \ref{lem:bretagnolle-huber}), there exists another game-dependent constant $\epsilon' > 0$ such that
\begin{align*}
R_n(\theta_n) + R_n(\theta'_n) \geq n^{2/3} \epsilon' \exp\left(-2n^{-2/3} \EE_{\theta_n}[T_n(\pP(\theta)^c)] \norm{v}^2_G\right) \geq  n^{2/3} \epsilon' \exp(-2 \norm{v}^2_G)\,.
\end{align*}
Combining the last two displays completes the proof.
\end{proof}

\begin{theorem}
Suppose the game is locally observable, then there exists a constant $c > 0$ such that for all $n$ there is a $\theta$ for which $R_n(\theta) \geq cn^{1/2}$.
\end{theorem}

\begin{proof}
Let $\theta \in \RR^d$ be arbitrary and $\theta_n = n^{-1/2} \theta$ and $\theta_n' = -\theta_n$.
By the assumption that $\xX$ spans $\RR^d$, it follows that $\pP(\theta_n) \cap \pP(\theta_n') = \emptyset$.
By Lemma \ref{lem:kl},
\begin{align*}
\KL(\bbP_{\theta_n}^n, \bbP_{\theta_n'}^n) = \frac{1}{2} \norm{\theta_n' - \theta_n}^2_{V_n(\theta_n)} 
= \frac{1}{2} \norm{\theta}^2_{V_n(\theta_n)/n}\,.
\end{align*}
Clearly, $G = \sum_{x \in \xX} A_x A_x^\top \succ V_n(\theta_n) / n$.
Hence, there exists a constant $c > 0$ such that for all $n \geq 1$, 
\begin{align*}
\KL(\bbP_{\theta_n}^n, \bbP_{\theta'_n}^n) \leq c\,.
\end{align*}
Then, using the same argument as in the proof of \cref{thm:linear}, we have
\begin{align*}
\bbP_{\theta_n}^n(E_n(\theta_n)) + \bbP_{\theta_n'}^n(E_n(\theta_n')) 
\geq \bbP_{\theta_n}^n(E_n(\theta_n)) + \bbP_{\theta_n'}^n(1-  E_n(\theta_n)) 
\geq \frac{1}{2} \exp(-c)\,. 
\end{align*}
The result follows because there exists an $\epsilon > 0$ such that $R_n(\theta) \geq \bbP_{\theta}^n(E_n(\theta)) \epsilon \sqrt{n} / 2$.
\end{proof}

\end{document}